\documentclass[11pt]{article}

\usepackage[margin=1in]{geometry}
\usepackage[hyphens]{url}
\usepackage{graphicx}
\usepackage{booktabs}
\usepackage{amsmath,amssymb,amsfonts}
\usepackage{amsthm}
\usepackage{array}
\usepackage{microtype}
\usepackage{xcolor}
\usepackage{tikz}
\usetikzlibrary{arrows.meta,fit,backgrounds,positioning}
\usepackage[numbers,sort&compress]{natbib}
\usepackage{hyperref}

\hypersetup{
  colorlinks=true,
  linkcolor=blue!55!black,
  citecolor=green!35!black,
  urlcolor=blue!60!black,
  pdftitle={ICM Out! Better Tournament Strategy from Computed Continuations,
    vs. Solvers and LLMs},
  pdfauthor={Boning Li and Longbo Huang}
}
\newtheorem{proposition}{Proposition}
\newtheorem{lemma}{Lemma}

\newtheorem{remark}{Remark}

\DeclareRobustCommand{\sco}{\textnormal{\textsc{SCO}}}
\newcommand{\icm}{\mathrm{ICM}}
\newcommand{\scont}{\mathrm{SC}}
\newcommand{\sym}{\mathrm{SYM}}
\newcommand{\btn}{\textsc{btn}}
\newcommand{\sbpos}{\textsc{sb}}
\newcommand{\bbpos}{\textsc{bb}}

\title{ICM Out! Better Tournament Strategy from Computed Continuations, vs.\ Solvers and LLMs}

\author{Boning Li\thanks{IIIS, Tsinghua University. Email: \texttt{li-bn22@mails.tsinghua.edu.cn}.}
\and Longbo Huang\thanks{IIIS, Tsinghua University. Email: \texttt{longbohuang@tsinghua.edu.cn}. Corresponding author.}}

\date{August 2026}

\begin{document}
\maketitle

\begin{abstract}
The Independent Chip Model (ICM) converts tournament chips into reference prize
equity, and policies are routinely constructed against those values.
Because ICM reads only stack sizes, it omits action order, blind obligations, and
seat rotation, and it does not price the elimination pressure a big stack puts on
the short stacks it can bust. Those omissions can alter the successor-state
contrasts that determine a move. We introduce \emph{Strategic-Continuation Optimization} (\sco),
a policy-construction method that enumerates current-hand outcomes, maps them to
successor states, prices those states with continuation values computed from the
finite tournament model, and optimizes and freezes the resulting current-hand
policy. The fixed-ICM comparison policy changes one thing only: the same optimizer
solves the same game with successor states priced by analytic ICM, so the two
policies differ only through that pricing. We evaluate the resulting policies
in a three-player jam/fold tournament with a \$1{,}000{,}000 prize pool. Relative
to the frozen strategic-continuation benchmark, analytic ICM has \$9{,}036 mean
absolute value error across all 2,838 state--seat entries. That value error rewrites
the ranges it prices: measured against each decision point's own fixed-ICM jam range,
\sco\ moves the jam frequency by an average of 14.08\%, and by 32.42\% at the button.
To price those different moves, we compare all 946 states and three policy owners
while changing only the focal policy and holding both opponents and the continuation
evaluator fixed. The policy produced by \sco\ earns \$214.33 more
prize equity per hand on average and is favored in 2,433 of 2,838 matched units.
The ordering survives replacing the solver-built opponent with two Large Language
Models (LLMs) and with a family of non-modeling threshold players, neither of which
has any access to the tournament computation.
This value-to-policy-to-cost chain shows directly when a stack-only equity model
becomes an inadequate objective for tournament strategy construction.

\end{abstract}

\section{Introduction}
\label{sec:intro}

A stack of tournament chips is not proportional to prize money. Doubling a stack
does not double what it returns, because prizes are paid by finishing rank rather
than by chip count: the leader's equity is capped by first prize, while a short
stack still holds a claim on the ladder. Players therefore need a way to translate
a stack of chips into the money it is expected to return. The Independent Chip Model (ICM)
is the standard tool for this translation: it treats each player's probability
of finishing first, second, and so on as proportional to chip share, and reads
off an expected prize from the payout
schedule~\citep{harville1973assigning,diaconis2022gambler}. Because it depends
on nothing but the vector of stacks, ICM is cheap to compute and has become the
default reference for late-tournament equity.

As a description, that is defensible. Chips really are the dominant
determinant of prize equity, and ICM captures the dominant part. The trouble is
that ICM values do not only describe a position; they \emph{drive decisions}. A
jam/fold player compares jamming against folding by looking up the continuation
equity ICM assigns to each outcome and takes the higher lookup. A model can be a
good summary of a state and still be a poor thing to optimize against, because
what a controller needs is not the level of the value but the \emph{differences}
between the values of the states it can move to.

We work in a finite three-player jam/fold tournament, small enough to enumerate
every stack configuration and to freeze and evaluate full-state policy tables,
and already carrying the structure a stack-only table discards: action order,
blind obligations, and seat rotation all shape what a position is worth. We
state every quantity in the money of a \$1{,}000{,}000 prize pool, the scale of a
major live final table. As in a cash game, position is worth money here, and ICM
cannot see any of it: a continuation table computed from the tournament itself
ranks the button above the small blind and the small blind above the big blind,
separating the button from the big blind by \$18{,}595 at equal 45-chip stacks.
Beyond position, ICM also fails to price the pressure of elimination. A chip
leader can jam without risking its own tournament life, while the short stack it
covers must call for everything or fold; ICM reads the same chip vector and prices
neither side of that threat. Our question is what that does to play:
\emph{when do ICM value errors change the move, and what do the changed moves
cost against solver and Large Language Model (LLM) opponents?}

We answer it with a policy-construction method and three measurements. The method,
\emph{Strategic-Continuation Optimization} (\sco), builds a policy by optimizing
current-hand play against a strategic continuation table computed from the
tournament itself. We then measure, in order, how far analytic ICM sits from a
frozen Strategic-Continuation (SC) benchmark across all state--seat entries; how far
the induced policy $\pi^{\sco}$ sits from $\pi^{\icm}$, built the same way against
analytic ICM values; and what that policy difference is worth in prize equity. The
third measurement is matched: for each state and focal policy owner we replace only
that player's policy, hold the opponents fixed, and score both arms with the same
continuation evaluator, so the paired difference in Prize Equity (PE), $\Delta PE$,
is attributable to the
focal policy and nothing else, across all 2,838 state--owner units of the complete
domain.

Across all 2,838 entries, analytic ICM minus the frozen SC benchmark has mean
absolute error \$9{,}036. Those prices rewrite the ranges they govern. Scaled
against each decision point's own fixed-ICM jam range, the induced jam frequency
moves by an average of 14.08\% across the 5,676 state--information-set cells, and by
32.42\% at the button. The matched cost is then direct: $\pi^{\sco}$ earns \$214.33
more per hand, and is favored in 2,433 of the 2,838 units, in 903 of the 946 states,
at every seat, and in every leave-one-out aggregate we can form. Averaging the
strategic table over the six seat assignments of each chip multiset splits that
gain exactly, with \$0 residual, into a \textbf{positional} term worth \$130.43
(60.86\%), the value of knowing which seat holds which stack, and a \textbf{level}
term worth \$83.90 (39.14\%), the value of pricing chips from the tournament itself
even for a table that never looks at seats. ICM misses both: what position is worth,
and how the tournament's own dynamics price a stack.

The result belongs to the play rather than to the machinery that measured it. Six
continuation evaluators, five computed independently of the table that produced
$\pi^{\sco}$, all keep the sign, as does a Monte Carlo evaluator that abandons the
frozen tables and deals three disjoint hands and a board from one shared deck. Two
LLMs supply opponent profiles built from rules and observable information alone, and
a seven-member family of non-modeling threshold opponents covers a wide sweep of
opponent aggression. Removing the continuation lookup entirely and playing the
tournament out to a finish, 20,000 paired replicates per unit over 496{,}434{,}460
three-way hands, raises the gain to \$938.03 per hand, on a pooled standard error
of \$18.33.

\paragraph{Contributions.}
\begin{enumerate}\itemsep2pt
\item We locate what ICM misses across the complete finite domain: its values sit
\$9{,}036 from the tournament's own continuation values in mean absolute error,
it prices neither the seat a stack sits in nor the elimination pressure a covering
stack applies, and the resulting jam ranges move by 14.08\% on average and by
32.42\% at the button.
\item We propose Strategic-Continuation Optimization, which optimizes current-hand
play against a continuation table computed from the tournament itself, and show
under a matched protocol that it earns \$214.33 per hand more than fixed ICM, rising
to \$938.03 per hand when the tournament is played out to a winner.
\item We split that gain exactly, with \$0 residual, into \$130.43 for knowing which
seat holds which stack and \$83.90 for pricing chips from the tournament's own
dynamics, and we hold the direction fixed across six continuation evaluators, a
shared-deck Monte Carlo evaluator, two LLM-driven opponent profiles, a non-modeling
opponent family, three chip depths, and two prize ladders.
\end{enumerate}

\section{Related Work}
\label{sec:related}

\paragraph{Tournament equity and ICM.}
ICM uses the sequential-ranking form associated with Harville's model to convert
stacks into payout-weighted finishing probabilities~\citep{harville1973assigning}.
Its mathematical behavior and empirical fit have been studied through
probabilistic and tournament-outcome
analyses~\citep{diaconis2022gambler,kim2025empirical,scott2024money}, while
strategy-side treatments of late-stage jam/fold
play~\citep{sklansky1994holdem,sklansky2007tournament,carter2007investigation}
take the equity model as given. That literature evaluates ICM as an
\emph{estimator} of prize equity. We evaluate it as a \emph{controller}: not how
close its numbers are to the truth, but how much prize equity is lost by acting
on them. The two questions come apart, because a controller is sensitive only to
differences between the values of reachable successors, and a model can be
accurate on average while systematically misordering the comparisons a decision
actually turns on.

\paragraph{Tournament computation.}
Prior work extends heads-up jam/fold
analysis~\citep{DBLP:conf/atal/MiltersenS07} to three-player fixed-point methods
over tournament-equity tables. Prior computations enumerate the 946 ordered stack
states and either compare equilibrium payoffs with ICM
predictions~\citep{ganzfried2008computing} or evaluate deviations through an
induced MDP~\citep{ganzfried2009computing}. Our experiment instead freezes the
policies, holds the evaluator and opponents fixed, replaces only the focal
owner's policy, and reports induced $\Delta PE$ over all 2,838 state--owner
units. Regret minimization and search have also produced strong agents in much
larger imperfect-information games~%
\citep{zinkevich2007regret,bowling2015heads,moravvcik2017deepstack,brown2018superhuman,brown2019superhuman},
including work that carries them past the two-player zero-sum case where the
guarantees hold~\citep{abou2010using,gibson2011regret,gibson2011strategy,%
szafron2013parameterized,games/GanzfriedNP18} and work that makes the underlying
solve faster or sharper~\citep{li2024rl,li2025efficient,li2026real,%
li2026correlatedchancesamplingmonte}. We contribute an evaluation
design and an empirical decomposition rather than a new solver.

\paragraph{Value functions as controllers.}
The distinction we draw is familiar outside poker. In approximate dynamic
programming the quantity that governs greedy-policy quality is not the accuracy
of a value function but the error in the \emph{advantages} it induces, and a
uniformly biased value function can be a perfect controller while a
small-average-error one is a poor one. Our seat/level decomposition is an
instance of that phenomenon with an unusually clean structure: because ICM is
invariant under the same seat-averaging operator applied to the strategic table, the loss splits
into the part attributable to an unmodeled state variable and the part
attributable to miscounting within the variables the model does see, with no
interaction residual.

\paragraph{Behavioral model opponents.}
Where prior work models an opponent to exploit
it~\citep{billings1998opponent,southey2005bayes,hoehn2005effective,%
li2026agentscertifyexploitsconfidencescheduled} or asks how
well a language model itself plays
poker~\citep{DBLP:journals/corr/abs-2308-12466,DBLP:journals/corr/abs-2401-06781,DBLP:journals/corr/abs-2309-17277,li2026pokerskill},
language models here instantiate two fixed opponent profiles built outside the
tournament computation. Each maps an information-restricted prompt carrying the
rules and the observable decision information to a jam/fold action tensor that
both primary hero arms then face identically. Prompts ask for an action, never for
equity, and reveal neither hero policy, which makes each profile an opponent the
two arms meet on equal terms rather than a source of values.

\section{Problem Setup and Estimand}
\label{sec:setup}

\paragraph{Tournament model.}
We study a finite three-player tournament whose \$1{,}000{,}000 prize pool pays
\$750{,}000, \$250{,}000, and \$0 for first, second, and third. All
equities and gains below are stated in that money, so \$1 is $10^{-6}$ of the
pool. A state $s=(a,b,c)$ records positive integer stacks at the button (\btn),
small blind (\sbpos), and big blind (\bbpos), with $a+b+c=T$. The main grid uses
$T=45$, blinds $1/2$, and all $\binom{44}{2}=946$ ordered states. The finite model uses a jam/fold action space and the standard 169 preflop
classes~\citep{gilpin2007lossless,sandholm2010state,li2026effective}, and positions rotate after
each hand. Every state and hand class in this model is enumerated exactly. Posted
blinds, folds, called jams, chip transfer, elimination, and seat rotation are
integrated into the current-hand payoff.

\paragraph{Continuation tables.}
A hand does not end the tournament; it produces a new stack configuration from
which play continues. A \emph{continuation table} $C$ captures the value of that
future in a single lookup: it maps a successor stack vector to the expected prize
money each surviving player will eventually collect. The two policy-construction
procedures use a continuation table in exactly this way and differ only in the
table used to solve their current-hand games.

ICM continuation, denoted $C^{\icm}$, is the Malmuth--Harville
table~\citep{harville1973assigning,malmuth1999gambling}. It assigns
finishing-order probabilities by drawing players out one at a time in proportion
to their chips: for stacks $(x_1,\dots,x_n)$ and payouts $(v_1,\dots,v_n)$,
\begin{equation}
C^{\icm}_i=\!\!\sum_{\text{orders }\rho}\!\Bigl(\textstyle\prod_{k}\frac{x_{\rho(k)}}{\sum_{\ell\ge k}x_{\rho(\ell)}}\Bigr)\,v_{\text{rank}_\rho(i)},
\label{eq:icm}
\end{equation}
summing over all finishing orders $\rho$. Equation~\eqref{eq:icm} reads the stack
vector and nothing else, so it is invariant to which seat holds which stack.
Strategic continuation, denoted $C^{\scont}$, is instead computed from the finite
tournament model itself, so it retains the relationship between stack ownership,
action order, and the blinds each seat must post next. The two tables agree that
more chips are better; they disagree about what a position is worth once seat and
order are taken into account. Throughout, $C^{\scont}$ is the frozen benchmark
inside this finite model. It incorporates the model's jam/fold dynamics, payouts,
blinds, and continuation convention; it is not an absolute ground-truth value for
richer tournament poker.

For an entry $(s,i)$, define benchmark-relative value error
\begin{equation}
e_i(s)=C_i^{\icm}(s)-C_i^{\scont}(s).
\label{eq:value-error}
\end{equation}
The census of $e_i(s)$ uses all 946 states and three absolute positions, with no
mask. Its absolute magnitude measures disagreement between two continuation
tables. It does not by itself determine a move.

\paragraph{When value error changes the move.}
At any information set and hand $h$, let
$m^{\scont}(h)=Q^{\scont}_{\rm jam}(h)-Q^{\scont}_{\rm fold}(h)$ and define the
contrast error
\begin{equation}
\delta(h)=\bigl(Q^{\icm}_{\rm jam}-Q^{\icm}_{\rm fold}\bigr)(h)-m^{\scont}(h).
\label{eq:contrast-error}
\end{equation}
For pure action selection away from ties, the preferred move differs exactly when
$m^{\scont}(h)$ and $m^{\scont}(h)+\delta(h)$ have opposite signs. Thus a common
level shift can be large and action-irrelevant, while a smaller successor-specific
error can reverse the pure preferred move by crossing the margin. Smoothed policies
can change their jam probability as the same contrast changes, without requiring a
sign crossing, and are especially sensitive near zero. Our policy
statistics therefore measure the induced difference directly rather than infer it
from $|e_i(s)|$. Aggregate value error, policy difference, and matched cost are
separate estimands, not a statewise monotonicity or mediation claim.

\paragraph{Two ways to be wrong, visible in the values.}
The cleanest illustration is a symmetric stack. Under Eq.~\eqref{eq:icm}, three
equal stacks price every seat at exactly \$333{,}333. Strategic continuation does
not, because the seats are not interchangeable in the next hand
(Table~\ref{tab:equal-stack}). The button posts no blind and can therefore fold
at no cost, so it is worth the most; the big blind is already partially committed,
so it is worth the least.

\begin{table}[tb]
\centering
\small
\begin{tabular}{@{}lcccc@{}}
\toprule
Equal-stack state & \btn\ (\$) & \sbpos\ (\$) & \bbpos\ (\$) & mean (\$) \\
\midrule
$T=30$: $(10,10,10)$ & 347{,}328 & 337{,}450 & 315{,}220 & 333{,}333 \\
$T=45$: $(15,15,15)$ & 341{,}568 & 335{,}458 & 322{,}973 & 333{,}333 \\
$T=60$: $(20,20,20)$ & 339{,}875 & 334{,}080 & 326{,}043 & 333{,}333 \\
\midrule
ICM (any depth) & 333{,}333 & 333{,}333 & 333{,}333 & 333{,}333 \\
\bottomrule
\end{tabular}
\caption{Strategic continuation equity at equal stacks, in dollars of a
\$1{,}000{,}000 prize pool. The strict $\btn>\sbpos>\bbpos$ ordering holds at all
three depths, and at 45 chips the button and big blind are \$18{,}595 apart. The
seat mean coincides with ICM here \emph{because the state is symmetric}: the
three seats exhaust one chip multiset, so averaging them recovers the ICM value.}
\label{tab:equal-stack}
\end{table}

The last column of Table~\ref{tab:equal-stack} carries a second message. At a
symmetric state the three seat values average to exactly the ICM value, so seat
blindness is the \emph{only} error there. Symmetry is what makes that average
work, and Appendix~\ref{app:sym-equal} shows it holds at this state alone. At a general state the six seat assignments of a chip multiset are
different games, and their strategic average need not equal Eq.~\eqref{eq:icm}.
Section~\ref{sec:decomposition} constructs that average explicitly and finds
that its comparison with ICM contributes \$83.90 per hand along the chosen
symmetrization path: even a seat-blind strategic table beats ICM. This is the
second table difference, and equal-stack intuition hides it.

\paragraph{Policy arms and estimand.}
Strategic-Continuation Optimization (\sco) constructs a policy by optimizing the
current-hand game against a strategic continuation table. Its frozen output
$\pi^{\sco}$ is produced against $C^{\scont}$. The comparison policy
$\pi^{\icm}$ uses the same current-hand optimization with analytic $C^{\icm}$ held
fixed. Both policies are frozen before evaluation; in particular, $\pi^{\sco}$ is
not adapted to either language model. We call these the two primary policy arms
when they are inserted as the focal policy in the matched evaluation.

For state $s$, focal owner $i$, opponent profile $o_{-i}$, and common evaluation
continuation $C$, define
\begin{equation}
\Delta PE(s,i;o,C)=
U_i(\pi_i^{\sco},o_{-i};C)-
U_i(\pi_i^{\icm},o_{-i};C),
\label{eq:delta-pe}
\end{equation}
where $U_i$ is expected tournament prize money after exactly the current hand and
continuation lookup. A positive value means that replacing the focal owner's
fixed-ICM policy with the policy $\pi^{\sco}$ produced by \sco\ earns more prize equity. The opponents,
evaluator, and continuation convention are identical across the two terms, so the
difference isolates the focal policy.

\paragraph{A current-hand optimization benchmark.}
The matched gain can also be located relative to a common current-hand optimum.
Let $\Pi_i^{\rm hand}$ denote the focal player's behavioral strategies in the
modeled current jam/fold hand. Opponents and the continuation evaluator remain
fixed, and no future-state policy is re-optimized.

\begin{proposition}[Current-hand optimization boundary]
\label{prop:boundary}
Fix $(s,i,o,C)$ and suppose $\pi_i^{\sco}$ is $\varepsilon$-optimal over
$\Pi_i^{\rm hand}$ under $C$ against $o_{-i}$:
\begin{equation}
U_i(\pi_i^{\sco},o_{-i};C)\geq
\sup_{\pi_i\in\Pi_i^{\rm hand}}U_i(\pi_i,o_{-i};C)-\varepsilon.
\label{eq:optimization-boundary}
\end{equation}
Then $\Delta PE(s,i;o,C)\geq-\varepsilon$ for every comparison policy
$\pi_i^{\icm}\in\Pi_i^{\rm hand}$, and exact optimality gives
$\Delta PE\geq0$. Define the current-hand optimization gap
\begin{equation}
g_i(\pi_i;s,o,C)=
\sup_{\pi_i'\in\Pi_i^{\rm hand}}U_i(\pi_i',o_{-i};C)
-U_i(\pi_i,o_{-i};C).
\label{eq:hand-gap}
\end{equation}
Then
\begin{equation}
\Delta PE(s,i;o,C)=g_i(\pi_i^{\icm};s,o,C)-g_i(\pi_i^{\sco};s,o,C).
\label{eq:ceiling}
\end{equation}
\end{proposition}

\begin{proof}
$\pi_i^{\icm}$ is feasible in the supremum in
Eq.~\eqref{eq:optimization-boundary}, so
$U_i(\pi_i^{\icm},o_{-i};C)\le
\sup_{\pi_i\in\Pi_i^{\rm hand}}U_i(\pi_i,o_{-i};C)\le
U_i(\pi_i^{\sco},o_{-i};C)+\varepsilon$, which is the sign claim. For the
identity, both gaps use the same supremum, which cancels in their difference.
\end{proof}

Under the continuation that defines this current-hand game, the optimization
boundary explains the sign up to the solve tolerance; magnitude and coverage
remain empirical. Changing $C$ changes the scored game and voids that boundary,
which identifies the intervention tested in Appendix~\ref{app:robustness}.

\section{Method}
\label{sec:method}

\paragraph{Estimating a player's prize equity.}
The heart of the comparison is a single quantity: the expected prize a player
takes away from a state, playing one hand and then continuing. For the census, we
compute it by deterministic aggregation over a frozen discrete payoff table for
every state, and Figure~\ref{fig:method} shows the four steps.

\paragraph{From a state to a current hand.}
A state fixes the three stacks and the seats. Each seat is dealt one of the 169
preflop hand classes and chooses, according to its policy, whether to jam all-in
or fold. Seats too short to cover a blind are treated as already committed, so
the hand resolves under the true blind structure.

\paragraph{From decisions to terminal outcomes.}
The jam/fold decisions of the three seats generate a distribution over terminal
outcomes. If everyone folds to a blind, chips move by the posted amounts. If one
player jams and the others fold, the jammer takes the posted blinds uncontested.
If exactly one player jams and is called, the two contest a heads-up all-in and
the third keeps its remaining stack. If more players commit, the hand becomes a
two- or three-way all-in that is settled with a main pot and, when stacks are
unequal, a side pot. Each outcome is weighted by the hand-class reach
probabilities and the policy probabilities that produced it.

\paragraph{From terminal chips to prize shares.}
Every terminal outcome leaves a successor stack vector $s'$, and the continuation
table prices it. When all three players survive, $s'$ is looked up directly in the
three-player table. When one player is eliminated, that player collects the
third-place prize and the two survivors are priced by the heads-up table. When two
players bust in the same hand, the survivor takes first prize and the two
eliminated seats are ranked by their stacks at the start of the hand. This step is where a continuation table prices terminal outcomes. During policy
construction the two arms use different tables; during matched scoring both arms
use the same frozen $C^{\scont}$, so the focal policy rows are the only difference.

\paragraph{From outcomes to expected equity.}
Summing the priced outcomes against their reach and policy weights gives $U_i$,
the prize money player $i$ expects after exactly one hand and one continuation
lookup. This is the quantity in Eq.~\eqref{eq:delta-pe}: a current-hand
expectation in which the continuation table stands in for all future hands.
Section~\ref{sec:rollout} relaxes exactly this by playing the tournament out to a
finish.

\paragraph{A worked case.}
Consider the equal-stack state $(15,15,15)$. After any no-elimination outcome,
positions are re-indexed: the current button becomes the next big blind, the
current small blind becomes the next button, and the current big blind becomes the
next small blind. Thus a successor written in current-hand order as
$(x'_{\btn},x'_{\sbpos},x'_{\bbpos})$ is looked up in next-hand order as
$(x'_{\sbpos},x'_{\bbpos},x'_{\btn})$. ICM assigns one value to a stack multiset
however it is assigned; strategic continuation distinguishes which player inherits
each next-hand position. Called-jam branches settle the showdown and any
elimination before this lookup, and the matched comparison evaluates both policies
under the same frozen continuation.

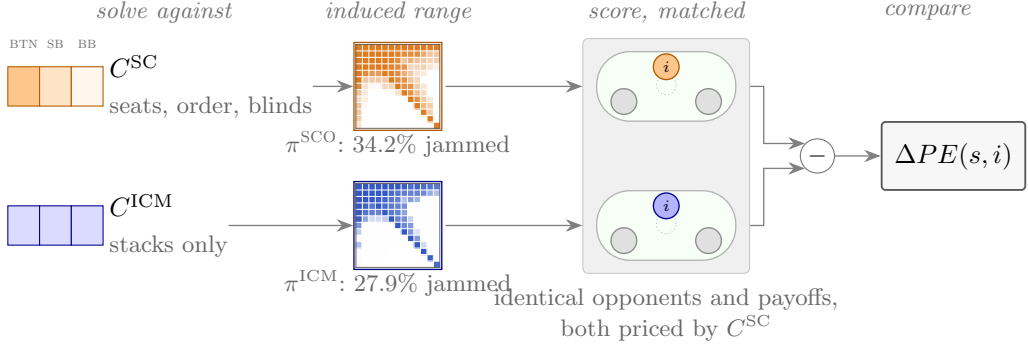
\begin{figure*}[t]
\centering
\newcommand{\scoredhand}[4]{%
  \begin{scope}[shift={(#1,#2)}]
    \draw[draw=black!30, fill=green!3, line width=0.5pt, rounded corners=4.4mm]
      (-0.92,-0.46) rectangle (0.92,0.46);
    \draw[draw=black!25, densely dotted, line width=0.4pt] (0,0.02) circle (0.13);
    \node[seat, draw=black!45, fill=black!12] at (-0.56,-0.20) {};
    \node[seat, draw=black!45, fill=black!12] at ( 0.56,-0.20) {};
    \node[seat, draw=#4, fill=#3, font=\tiny] at (0,0.26) {$i$};
  \end{scope}}
\begin{tikzpicture}[
  font=\small,
  >={Stealth[length=2.2mm]},
  x=1cm, y=1cm,
  cell/.style={line width=0.4pt, minimum width=4.2mm, minimum height=5.2mm,
               inner sep=0pt},
  seat/.style={circle, line width=0.5pt, minimum size=3.4mm, inner sep=0pt},
  note/.style={font=\footnotesize, text=black!60, align=center},
  colname/.style={font=\tiny, text=black!55},
  blk/.style={anchor=west, align=left, inner sep=0pt},
  stage/.style={font=\footnotesize\itshape, text=black!55},
  flow/.style={->, draw=black!50, line width=0.5pt},
]

\def\ysc{0.92}
\def\yicm{-0.92}

\node[stage] at (-6.05,1.90) {solve against};
\node[stage] at (-2.95,1.90) {induced range};
\node[stage] at (0.60,1.90)  {score, matched};
\node[stage] at (4.05,1.90)  {compare};

\node[colname] at (-7.90,1.48) {\btn};
\node[colname] at (-7.48,1.48) {\sbpos};
\node[colname] at (-7.06,1.48) {\bbpos};

\foreach \x/\c in {-7.90/orange!45, -7.48/orange!22, -7.06/orange!7}
  \node[cell, draw=orange!70!black, fill=\c] at (\x,\ysc) {};
\node[blk] (tsc) at (-6.77,\ysc)
  {$C^{\scont}$\\[1.5pt] {\footnotesize\color{black!60}seats, order, blinds}};
\node[inner sep=0.6pt, draw=orange!70!black, line width=0.5pt] (rsc)
  at (-2.95,\ysc) {\includegraphics[width=11.2mm]{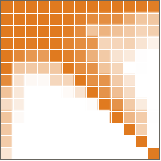}};
\node[note] at (-2.95,0.16) {$\pi^{\sco}$: 34.2\% jammed};

\foreach \x in {-7.90,-7.48,-7.06}
  \node[cell, draw=blue!55!black, fill=blue!14] at (\x,\yicm) {};
\node[blk] (ticm) at (-6.77,\yicm)
  {$C^{\icm}$\\[1.5pt] {\footnotesize\color{black!60}stacks only}};
\node[inner sep=0.6pt, draw=blue!55!black, line width=0.5pt] (ricm)
  at (-2.95,\yicm) {\includegraphics[width=11.2mm]{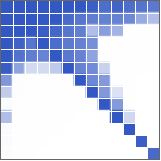}};
\node[note] at (-2.95,-1.68) {$\pi^{\icm}$: 27.9\% jammed};

\scoredhand{0.60}{\ysc}{orange!45}{orange!70!black}
\scoredhand{0.60}{\yicm}{blue!30}{blue!55!black}
\node[inner sep=0pt, minimum width=18.4mm, minimum height=9.0mm] (hsc)
  at (0.60,\ysc) {};
\node[inner sep=0pt, minimum width=18.4mm, minimum height=9.0mm] (hicm)
  at (0.60,\yicm) {};
\begin{scope}[on background layer]
  \node[fit=(hsc)(hicm), inner sep=5pt, rounded corners=2pt,
        fill=black!6, draw=black!22, line width=0.4pt] (band) {};
\end{scope}
\node[note] at (0.60,-2.10)
  {identical opponents and payoffs,\\[0.5pt] both priced by $C^{\scont}$};

\node[draw=black!55, circle, inner sep=1.2pt, font=\footnotesize] (mi) at (2.60,0)
  {$-$};
\node[draw=black!65, fill=black!3, rounded corners=1.5pt, line width=0.7pt,
      minimum width=19mm, minimum height=9mm] (d) at (4.40,0) {$\Delta PE(s,i)$};

\draw[flow] (tsc.east)  -- (rsc.west);
\draw[flow] (ticm.east) -- (ricm.west);
\draw[flow] (-2.32,\ysc)  -- (band.west |- hsc);
\draw[flow] (-2.32,\yicm) -- (band.west |- hicm);
\draw[flow] (band.east |- hsc)  -- ++(0.18,0) |- (mi.135);
\draw[flow] (band.east |- hicm) -- ++(0.18,0) |- (mi.225);
\draw[flow] (mi.east) -- (d.west);

\end{tikzpicture}
\caption{The two primary arms are solved against their own continuation table, then frozen and scored
alike. At the illustrated equal-stack state, the strategic table (orange) prices
the seats differently while analytic ICM (blue) repeats the same value across the
three seats; in general, ICM is invariant to seat relabeling for a fixed stack
assignment, not equal across unequal stacks. The charts are the
induced button jam ranges at $(15,15,15)$ on the $13\times13$ preflop grid, shaded
by jam probability. Their difference illustrates the middle link of the evidence
chain: continuation errors affect action through jam--fold successor contrasts; pure
action reversal requires a zero crossing, while smoothed probabilities can change
without one. Only the focal owner $i$ changes between
the two scored hands, so
$\Delta PE(s,i)$ isolates its policy. Note that collapsing the orange row to its
own average would produce a third table, seat-blind but not equal to ICM ---
that table is the object of Section~\ref{sec:decomposition}.}
\label{fig:method}
\end{figure*}

\paragraph{Strategic-Continuation Optimization.}
Strategic-Continuation Optimization (\sco) maps a strategic continuation table to
a current-hand policy: for every full state, it solves the jam/fold game whose
terminal outcomes are priced by that table, then freezes the resulting policy. In
our finite model, the strategic table $C^{\scont}$ is supplied by the established
Value-Iteration--Fictitious-Play (VI--FP) machinery of
\citet{ganzfried2009computing}: each outer value-iteration sweep solves every
current-hand game with smoothed fictitious play~\citep{brown1951iterative,leslie2006generalised}
and updates the continuation table. The baseline application of \sco\ to the
archived $C^{\scont}$ produces $\pi^{\sco}$. The comparison branch uses the same
current-hand optimization while holding analytic ICM continuation fixed,
producing $\pi^{\icm}$. Both optimizations range over the modeled current-hand
policy space with the other players and continuation values fixed.

The fixed-ICM current-hand solves use a \$250 inner tolerance over all 946 states
and all three policy owners. The strategic outer computation uses a \$500 local
acceptance level; its frozen table has an ex-post worst one-hand deviation gain of
\$364.68, while the fixed-ICM arm attains \$250.00. Section~\ref{sec:scale}
re-solves both arms in lockstep at three common tolerances, which is the direct
test of whether the gain is really shared residual.

\paragraph{Matched evaluation.}
The comparison must be matched, or a payoff difference could come from the scoring
rather than the policy. For every state $s$ and focal owner
$i\in\{\btn,\sbpos,\bbpos\}$, we construct two profiles that are identical except
in the focal owner's rows. One uses $\pi_i^{\sco}$ and the other uses
$\pi_i^{\icm}$; all rows owned by the other two players remain at the same frozen
$\pi^{\sco}$ anchor. Both profiles are then evaluated with the \emph{same}
strategic continuation $C^{\scont}$. Fixing the opponents and the evaluator means
the paired difference reflects only the focal owner's policy, which is exactly the
object Eq.~\eqref{eq:delta-pe} defines. Sweeping every state and owner produces
$946\times3=2{,}838$ paired $\Delta PE$ units.

\paragraph{Descriptive policy difference.}
Before pricing the policies, we describe how often their moves differ. For a binary
jam/fold row, Total Variation (TV) is the absolute jam-probability difference,
$\operatorname{TV}_{\rm bin}=|p_{\sco}(\mathrm{jam})-
p_{\icm}(\mathrm{jam})|$. The primary policy statistic averages this quantity
equally over states and information-set rows, using the unconditional uniform
prior over the 1,326 private-card combinations. Forced rows remain in the average.
The statistic asks how different the frozen tables are on their common indexed
domain.

Absolute percentage points understate how much of a jam range is rewritten, because
a decision point that jams 5\% of combinations and one that jams 80\% are put on the
same scale. We therefore also report the difference relative to each decision
point's own fixed-ICM jam range,
\begin{equation}
X(s,u)=\frac{|p_{\sco}(\mathrm{jam}\mid s,u)-p_{\icm}(\mathrm{jam}\mid s,u)|}
{p_{\icm}(\mathrm{jam}\mid s,u)},
\label{eq:relative-jam}
\end{equation}
averaged equally over the $946\times6=5{,}676$ state--information-set cells. Every
denominator is strictly positive here, the smallest fixed-ICM jam rate in the domain
being 0.097\%, so the average needs no truncation. We
also report the mean absolute cellwise difference in combo-prior marginal jam
rate, a uniform
169-hand-class version, near-pure opposite-action mass, and breakdowns by owner and
information set.

Three properties of this design fix what the census measures. There is no
sampling: the 2,838 units are the complete population at $T=45$, so the coverage
counts we report are census facts. The two arms share the settle logic bit for
bit, so identical policy rows produce identical values --- a property we assert as
a test rather than assume. And the anchor is the same in both terms, so the design
does not credit $\pi^{\sco}$ for facing a weaker opponent;
Section~\ref{sec:opponents} replaces the anchor outright to vary \emph{which}
opponent the two arms share.

\paragraph{What differs from ICM, and why it helps.}
The two arms part ways at a single point. The solver, the one-hand game, the
showdown arithmetic, and the evaluator are all shared. The only difference is the
continuation table each policy is solved against: $\pi^{\icm}$ prices successor
stacks with Eq.~\eqref{eq:icm}, while $\pi^{\sco}$ prices them with the
strategic table.

That single difference has a concrete consequence for play. ICM's price of a
successor depends only on the chip vector, so it cannot distinguish a player about
to post the big blind from one acting on the button with the same stack. Two spots
ICM treats as identical differ in continuation value once the next hand's blinds
and action order are counted, yet a policy optimizing the ICM price makes the same
jam/fold choice in both. This is why a seat-aware table beats a seat-blind one.
Whether ICM is even the right seat-blind table is a separate question, and
Section~\ref{sec:decomposition} shows it is not.

\section{The Complete Census}
\label{sec:census}

Every dollar quantity is stated in the money of the \$1{,}000{,}000 prize pool.
The census first records disagreement between analytic ICM and the frozen SC
benchmark, then the induced policy difference, and finally the matched gain
$\pi^{\sco}$ minus $\pi^{\icm}$.

\paragraph{Three-layer evaluation of \sco.}
The census reads \sco\ at three layers. The value layer measures
$C^{\icm}-C^{\scont}$ over all 2,838 state--seat entries and records disagreement
with the frozen finite-model benchmark. The policy layer measures binary total
variation between the two frozen tables, weighted equally over states and
information sets under the unconditional 1,326-combo prior. The cost layer changes
only the focal policy, holds the opponents and the $C^{\scont}$ evaluator fixed, and
prices the resulting play. A value discrepancy reverses a pure preferred move only
when its decision-relevant successor contrast crosses zero, and smoothed
probabilities can change without reversal, so the three layers have to be read in
sequence rather than substituted for one another.

The value-error distribution is substantial but heterogeneous. Seat-specific Mean
Absolute Error (MAE) is \$9,127.43 at the button, \$5,907.48 at the small blind, and \$12,073.13 at the
big blind. The corresponding policy difference is largest for button-owned rows:
combo-prior mean binary TV is 7.5164 percentage points for the button, 5.6198 for
the small blind, and 3.5013 for the big blind. Scaled by each decision point's own
fixed-ICM jam range, the same ordering sharpens, because the button also holds the
tightest fixed-ICM range at 32.5176\% of combinations: the mean relative jam-range
change of Eq.~\eqref{eq:relative-jam} is 32.4161\% at the button, against 10.8235\%
and 10.1351\% at the two blinds, for 14.0781\% overall. The relative change exceeds
20\% at 21.63\% of decision points and 10\% at 43.18\% of them. Trimming does not
carry the average: restricting to cells whose fixed-ICM jam rate is at least 2\%
leaves 14.0609\%, and dropping the largest one percent of relative changes leaves
12.2429\%. Uniform weighting over the 169 hand
classes gives 4.96 points of absolute TV overall, and 0.72\% of combo mass takes near-pure
opposite actions, using jam probabilities below 0.05 versus above 0.95. The six
information-set means range from 3.0319 to 7.5164 percentage points; Appendix
\ref{app:details} reports each row and the released decision-chain artifact stores
the unrounded values.

\paragraph{Headline matched cost.}
Across all 2,838 state--owner units, $\pi^{\sco}$ improves mean prize equity by
\$214.33 per hand. The comparison favors it in 2,433 of the 2,838 units; 277 units
are negative and 128 are exactly zero, the latter being states where the two
policies coincide on every hand class. Aggregating owners within a state first,
903 of the 946 states favor it. On the previously defined stable subset the mean
rises to \$262.80 with 1,151 of 1,314 units positive (Table~\ref{tab:census}).

\begin{table}[tb]
\centering
\footnotesize
\setlength{\tabcolsep}{4pt}
\begin{tabular}{@{}lrrrr@{}}
\toprule
Scope & Units & Mean (\$) & Median (\$) & $\Delta PE>0$ \\
\midrule
All units & 2,838 & 214.33 & 76.08 & 2{,}433 \\
By state & 946 & 214.33 & --- & 903 \\
Stable subset & 1,314 & 262.80 & 77.70 & 1{,}151 \\
\midrule
Button & 946 & 291.03 & 95.25 & 845 \\
Small blind & 946 & 237.63 & 77.85 & 826 \\
Big blind & 946 & 114.33 & 63.25 & 762 \\
\bottomrule
\end{tabular}
\caption{Current-hand prize-equity gain of the frozen policy $\pi^{\sco}$ produced by \sco\ over $\pi^{\icm}$,
in dollars of a \$1{,}000{,}000 prize pool. A unit is a state--owner pair, and the
2,838 units are the complete population at $T=45$. The stable row uses the frozen
canonical mask defined in Appendix~\ref{app:protocol}.}
\label{tab:census}
\end{table}

No single state or seat carries the result. Dropping any one state leaves a mean
between \$200.80 and \$214.58; dropping any one seat leaves \$175.98 (omit button),
\$202.68 (omit small blind), or \$264.33 (omit big blind). Per-seat means are
\$291.03, \$237.63, and \$114.33 --- largest for the seat whose next-hand position
moves the most and smallest for the already-committed big blind, which is the
pattern the missing positional information predicts.

\paragraph{The mean is small; the underlying gap is not.}
Taken alone, \$214.33 sounds negligible: it is 0.0257\% of the \$333{,}333 an equal
three-way seat is worth. Two features of the distribution explain why the average
understates the phenomenon.

First, the distribution is strongly right-skewed. The median gain is \$76.08,
roughly a third of the mean, so most states improve slightly while a minority
carries the average. The largest single unit gains \$36{,}944.30; the worst loses
\$134.48. That asymmetry --- an upper tail nearly three hundred times the size of
the worst case --- is itself informative: when ICM's blindness matters, it matters
by orders of magnitude more than when it helps. Second, the census average is
taken over a domain deliberately including the positionally mild states. At
$(15,15,15)$ the button and big blind differ by \$18{,}595 under strategic
continuation, about 87 times the mean gain. The average is small because most
states are positionally mild, not because the gap is.

Sorting states by chip-share spread gives a descriptive profile.

\begin{table}[tb]
\centering
\footnotesize
\setlength{\tabcolsep}{4pt}
\begin{tabular}{@{}lrrr@{}}
\toprule
Stack-spread quartile & States & Mean (\$) & $\Delta PE>0$ \\
\midrule
Q1 (most even) & 232 & 92.35 & 656 / 696 \\
Q2 & 237 & 149.38 & 628 / 711 \\
Q3 & 222 & 159.20 & 552 / 666 \\
Q4 (most lopsided) & 255 & 433.65 & 597 / 765 \\
\bottomrule
\end{tabular}
\caption{Gain by quartile of chip-share spread ($\max-\min$ of the three stack
shares). The mean rises almost fivefold from the most even to the most lopsided
quartile while coverage falls, so lopsided states are where the gain concentrates
and also where it is least uniform.}
\label{tab:quartiles}
\end{table}

The two columns move in opposite directions, and the tension is the point. Even
stacks give a small but very reliable gain: nearly every unit improves, by a
little. Lopsided stacks give a large but less reliable one. Averaged over the
census these combine into a mean driven by Q4 and a coverage count driven by
Q1--Q2. The 277 negative units also have above-average spread (mean spread
$0.518$ versus $0.483$ census-wide), so the losses sit where the gain is largest.

\paragraph{Location relative to a common current-hand optimum.}
Equation~\eqref{eq:ceiling} locates the direct matched gain relative to the same
current-hand unilateral optimum for both policies. Across all 2,838 units, the
fixed-ICM policy has a mean current-hand optimization gap of \$378.73, while the
SCO policy's mean gap is \$164.40. Their difference is the directly measured
\$214.33 matched gain. The identity holds numerically to
$\$7.5\times10^{-10}$ per unit.

\begin{table}[tb]
\centering
\footnotesize
\setlength{\tabcolsep}{4pt}
\begin{tabular}{@{}lrrr@{}}
\toprule
Scope & ICM gap (\$) & SCO gap (\$) & Matched gain (\$) \\
\midrule
All 2,838 units & 378.73 & 164.40 & 214.33 \\
\bottomrule
\end{tabular}
\caption{Mean optimization gaps to a common current-hand unilateral optimum and
their difference. Opponents and $C^{\scont}$ are fixed, and optimization ranges
over the focal player's modeled current-hand policy.}
\label{tab:share}
\end{table}

The common optimum places the two frozen policies on one scale: under the scorer
and opponents used in the census, the SCO policy sits closer to the same local
benchmark, and the distance between them is the \$214.33 matched gain itself.

\paragraph{Does the scoring machinery produce the result?}
The census is scored with a frozen 169-class continuation table, so we also
scored it with an evaluator that shares none of that machinery: three disjoint
hole-card pairs and a five-card board dealt from one shared 52-card deck, 10,000
legal deals per state on 120 confirmatory states, which removes the independence
approximation implicit in class-level tabulation. It agrees with the frozen-table
evaluator to $-\$3.65$ per state on a standard error of \$12.25, and gives \$171.18
per hand with 96 of 120 states positive and a state-cluster bootstrap interval of
$[\$131.85,\$211.95]$. Five further continuation endpoints, computed from different
initializations independently of the table that produced $\pi^{\sco}$, all keep the
sign and the per-seat ordering, within $1.00$ to $1.13$ times the published mean.
Appendix~\ref{app:robustness} reports these and the remaining scorer and anchor
interventions in full, including the one cell that reverses the sign, where both
arms are scored with analytic ICM and $\pi^{\icm}$ is then the policy optimized for
the scored game.

\begin{figure*}[tb]
\centering
\includegraphics[width=\textwidth]{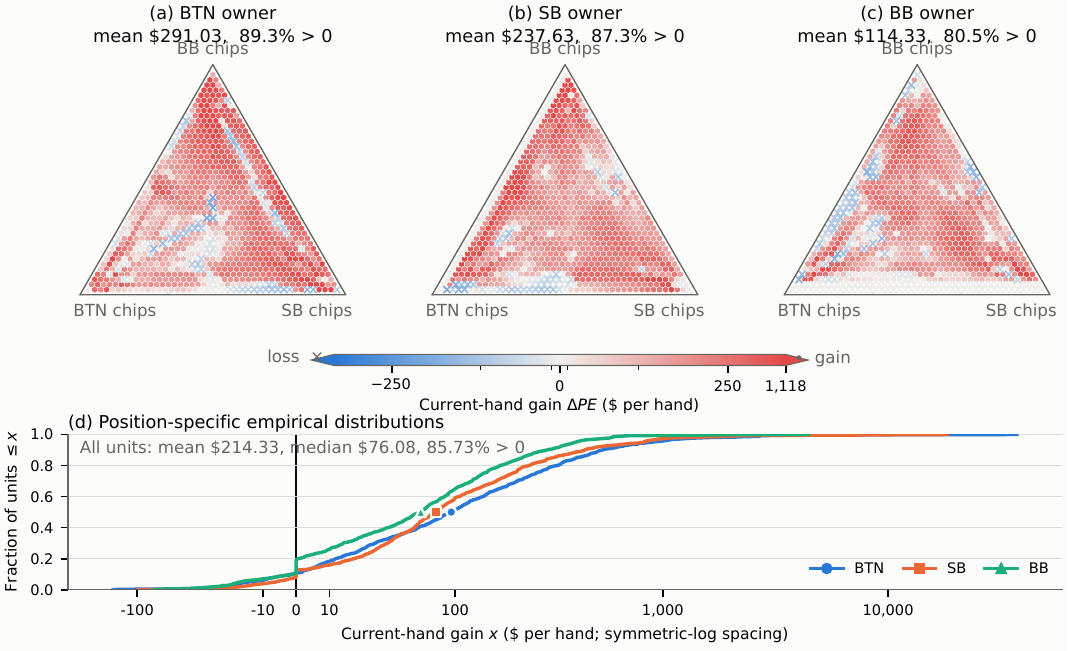}
\caption{The complete census, per state and per seat. Panels (a)--(c) place each
of the 946 chip triples at 45 total chips on the simplex whose corners are all
chips to the button, the small blind, and the big blind, colored by the gain to
the seat that owns the compared rows. Round markers show gains and crosses show
losses, so the sign reads without color, and the color scale is symmetric-log to
keep the small losses visible beside the right-skewed positive tail. The color
endpoints are the 98th percentile of $|\Delta PE|$ rather than the extremes: 57
of the 2,838 units saturate the scale, and the arrowheads mark that clipping.
Panel (d)
gives the three per-seat empirical distributions on the same symmetric-log axis
with each median marked; the vertical rule is zero. The losses are sparse, small,
and concentrated toward the lopsided edges of the simplex, while the large gains
sit in the interior bands where a seat's next-hand position is most in play.}
\label{fig:census}
\end{figure*}

\section{Where the Gain Comes From: An Exact Table Decomposition}
\label{sec:decomposition}

Section~\ref{sec:census} establishes three facts: ICM disagrees with the frozen
benchmark in value, induces a different policy, and loses \$214.33 per hand in the
matched comparison. This section splits that cost by a controlled table
intervention. The standard story is that ICM cannot see seats. A second table
difference remains after seat symmetrization, and the decomposition below shows
that both corresponding aggregate terms are positive along the chosen path.

\paragraph{An exact split.}
Let $M(s)$ be the multiset of chip counts in state $s$, and let $\Sigma(s)$ be the
six ordered states obtained by assigning those chips to the three seats. Define
the \emph{seat-symmetrized} table
\begin{equation}
C^{\sym}(s) \;=\; \frac{1}{|\Sigma(s)|}\sum_{s'\in\Sigma(s)} \Pi_{s'\to s}\,C^{\scont}(s'),
\label{eq:sym}
\end{equation}
where $\Pi_{s'\to s}$ realigns the seat indices of $s'$ onto those of $s$. By
construction $C^{\sym}$ depends only on $M(s)$: it is as seat-blind as ICM. But its
numbers are computed from the tournament, not from Eq.~\eqref{eq:icm}. Solving a
third policy arm $\pi^{\sym}$ against $C^{\sym}$ and scoring it in the same matched
design, define the aggregate utility
\begin{equation}
\bar U(\pi)=\frac{1}{2{,}838}
\sum_{\substack{s=(a,b,c):\,a+b+c=45\\a,b,c>0}}
\sum_{i\in\{\btn,\sbpos,\bbpos\}}
U_i(\pi_i,\pi^{\sco}_{-i};C^{\scont}).
\label{eq:aggregate-u}
\end{equation}
The three aggregate matched scores obey the identity
\begin{equation}
\underbrace{\bar U(\pi^{\sco})-\bar U(\pi^{\icm})}_{\text{total}}
=\underbrace{\bar U(\pi^{\sco})-\bar U(\pi^{\sym})}_{\text{positional}}
+\underbrace{\bar U(\pi^{\sym})-\bar U(\pi^{\icm})}_{\text{level}} .
\label{eq:decomp}
\end{equation}
The split is a telescoping sum, exact by construction, and we verify the residual
is \$0 exactly. It is also anchored: ICM is the fixed point of the same averaging
operator, since Eq.~\eqref{eq:icm} is already invariant to seat permutation, so
applying the averaging in Eq.~\eqref{eq:sym} to $C^{\icm}$ returns $C^{\icm}$.
$C^{\sym}$ therefore sits between the two arms on one axis: the positional term
removes seat information first, and the level term is the remainder between the
symmetrized strategic table and ICM.

\paragraph{Both terms are positive and neither is small.}

\begin{table}[tb]
\centering
\footnotesize
\setlength{\tabcolsep}{3.4pt}
\begin{tabular}{@{}lrrrrr@{}}
\toprule
Term & Mean (\$) & Share & Units $>0$ & States $>0$ & Seat means (\$) \\
\midrule
Positional & 130.43 & 60.86\% & 2{,}365 & 878 & 186.55 / 124.55 / 80.20 \\
Level & 83.90 & 39.14\% & 1{,}933 & 724 & 104.48 / 113.08 / 34.13 \\
\midrule
Total & 214.33 & 100\% & 2{,}433 & 903 & 291.03 / 237.63 / 114.33 \\
\bottomrule
\end{tabular}
\caption{Exact decomposition of the census gain, out of 2,838 units and 946
states. Seat means are button / small blind / big blind. The residual of
Eq.~\eqref{eq:decomp} is \$0 exactly. $\pi^{\sym}$ is solved to a worst one-hand
deviation gain of \$250.00 with at most 1,551 inner iterations against a cap of
1,600, so no state reached its cap.}
\label{tab:decomp}
\end{table}

Table~\ref{tab:decomp} is the paper's central negative claim about ICM. Seat
blindness accounts for \$130.43, three-fifths of the gain. The other two-fifths ---
\$83.90 --- is available to a table that never looks at a seat. That term is
positive in 1,933 of 2,838 units, in 724 of 946 states, at all three seats, in
every leave-one-state-out aggregate (range \$81.85 to \$84.85), and in every
leave-one-seat-out aggregate (minimum \$69.30).

The practical reading is direct: keeping one value per chip multiset but computing
it from the tournament captures the \$83.90 level term without tracking seats at
all.

\paragraph{The two terms behave differently.}
Their seat profiles do not match. The positional term declines monotonically across seats, from $\$186.55$ at the
button to $\$80.20$ at the big blind. The level term does not follow that order. It
is largest at the small blind ($\$113.08$), slightly smaller at the button
($\$104.48$), and smallest at the big blind ($\$34.13$).

\paragraph{Why ICM's levels are wrong even seat-blind.}
Eq.~\eqref{eq:icm} builds finishing probabilities by drawing players out in
proportion to chips, one at a time, independently of how the chips actually move.
Two approximations survive seat-averaging. First, chip movement is not proportional
elimination: in a jam/fold game chips move in large, correlated jumps determined by
stack coverage, so a short stack that is covered by both opponents busts far more
often than its chip share implies. Second, ICM prices the field as if all three
players continue to play the same way at every future stack configuration, whereas
the strategic table's own future play changes with the configuration, and the value
of a chip multiset depends on that future play.

The level term prices these two approximations: it is the cost of replacing the
symmetrized strategic table with ICM along the decomposition path, measured in
prize equity collected in play. Appendix~\ref{app:dominance} prices the same two
approximations one successor at a time, at the states where they are largest, and
shows they are big enough there to make ICM prefer a dominated action.

\section{Opponents Built Outside the Computation}
\label{sec:opponents}

Every opponent so far came from the solver, so both primary arms have only met a
style of opposition the same machinery produced. This section replaces the opponent
with two language models and with a family of mechanical threshold players, neither
of which has any access to the tournament computation.

\paragraph{Two information-restricted model profiles.}
Qwen2.5:32b and Llama3.3:70b~\citep{bai2023qwen,grattafiori2024llama} each receive
the three-player rules, the payouts, the blinds, the observable stack state, the
public action history, the acting position, and their own hand. They receive
nothing that could leak the answer: no continuation values from either table,
neither hero policy, no strategy margins, no oracle labels, and not the hero's
private cards. The requested output is a single jam/fold action, never an equity
estimate. Parsed decisions become one fixed opponent tensor per model that both
primary hero arms then face identically. The formal runs produced 117,962 legal responses
per model with no parsing or query failures, on 120 confirmatory states holding 351
valid state--owner units (Appendix~\ref{app:protocol} gives the sampling protocol).

\begin{table}[tb]
\centering
\footnotesize
\setlength{\tabcolsep}{4pt}
\begin{tabular}{@{}lcc@{}}
\toprule
Quantity (\$ per hand) & Qwen2.5:32b & Llama3.3:70b \\
\midrule
State-equal mean $\Delta PE$ & 1{,}069.45 & 996.68 \\
State-cluster 95\% CI & [539.05, 1{,}658.13] & [516.93, 1{,}540.68] \\
States positive & 73 / 120 & 78 / 120 \\
Unit-equal mean & 1{,}090.50 & 1{,}016.10 \\
\midrule
ICM-cont.\ mean & 835.38 & 749.25 \\
ICM-cont.\ 95\% CI & [306.88, 1{,}414.80] & [272.15, 1{,}286.30] \\
\midrule
Holm $p_{\rm adj}$ (dimensionless) & $<2\!\times\!10^{-4}$ & $<2\!\times\!10^{-4}$ \\
\bottomrule
\end{tabular}
\caption{Confirmatory comparison against rule-conditioned information-restricted
model configurations, each with 120 state clusters, 351 valid state--owner units,
and 117,962 legal responses. Dollar rows are per hand at a \$1{,}000{,}000 prize
pool. Primary rows weight the 120 states equally, whereas the unit-equal row
weights the 351 valid state--owner units equally. The Holm-adjusted one-sided
$p$-value is dimensionless and sits at the resolution floor of 10,000 resamples,
because no resample of either model falls below zero. Each model is one fixed
configuration, evaluated on its own.}
\label{tab:llm}
\end{table}

Both configurations support the ordering (Table~\ref{tab:llm}), and neither a
single state nor a single seat carries it: every leave-one-state-out and
leave-one-seat-out aggregate across both models stays positive, spanning \$502.73
to \$1{,}769.58. These gains are several times the census mean because the
confirmatory sample deliberately targets states where the two policies disagree; a
state on which both prescribe the same action carries no signal about which is
better. What the experiment adds is the solver-external opponent.

The seat pattern repeats the census with one exception. Against Qwen the per-seat
means are \$2{,}161.83 at the button, \$1{,}391.83 at the small blind, and $-\$316.48$ at
the big blind; against Llama, \$1{,}840.15, \$1{,}415.45, and $-\$232.30$. The big blind is
negative in both, though 54 of 116 units are positive against Qwen and 55 of 116
against Llama, so the seat carries a heavy negative tail rather than a systematic
loss, and it is positive under both other evaluators, at \$114.33 on the census and
\$122.18 under the shared deck. Against these opponents the ordering rests on the
button and small blind, the two seats where pricing next-hand position pays most.

Replacing the scoring continuation with analytic ICM lowers both means but
preserves the sign and the intervals. Recomputing all six continuation
endpoints from the same frozen response tensors --- no new model queries --- leaves
all twelve state-equal gains positive with every interval excluding zero; the twelve
means span \$992.38 to \$1{,}073.25, a spread of 0.73\% of the baseline mean for Qwen
and 0.79\% for Llama.

\paragraph{A family of non-modeling opponents.}
Language models are still adaptive systems of a sort. So we also swept a family of
mechanical opponents that model nothing at all: jam the top $k$\% of hands by
preflop strength, for $k\in\{10,20,30,40,50,60\}$, plus a coin flip on every hand.
Realized jam mass matches each target to within $0.0072$. Each family member
replaces the anchor for both arms, so the comparison stays matched.

\begin{table}[tb]
\centering
\footnotesize
\setlength{\tabcolsep}{3.4pt}
\begin{tabular}{@{}lrrr@{}}
\toprule
Anchor & Mean gain (\$) & ICM hand gap (\$) & States $>0$ \\
\midrule
Jam top 10\% & $-100.80$ & 9{,}859.55 & 483 \\
Jam top 20\% & $-92.63$ & 5{,}013.50 & 487 \\
Jam top 30\% & $-25.25$ & 2{,}747.48 & 501 \\
Jam top 40\% & 70.80 & 1{,}994.68 & 509 \\
Jam top 50\% & 165.80 & 2{,}124.38 & 564 \\
Jam top 60\% & 258.10 & 2{,}680.93 & 604 \\
Coin flip & 232.98 & 3{,}790.13 & 567 \\
\midrule
Solver anchor & 214.33 & 378.73 & 903 \\
\bottomrule
\end{tabular}
\caption{Non-modeling threshold anchors, out of 946 states. The mean rises
monotonically in opponent looseness and crosses zero between 30\% and 40\%. The
ICM hand-gap column gives the fixed-ICM policy's distance to the best unilateral
payoff in the modeled current hand under each fixed anchor and the common scorer.}
\label{tab:anchors}
\end{table}

Against very tight mechanical opponents the strategic policy is \emph{worse}, by up
to \$100.80 per hand, while both policies leave enormous room on the table: facing
an opponent who jams only 10\% of hands, the fixed-ICM policy lies \$9{,}859.55
below the modeled current-hand unilateral optimum, twenty-six times the
solver-anchor gap. Neither policy was built to exploit a mechanical opponent, so
their difference is small next to the room a specialized policy would claim.

The trend is the substantive finding. The gain increases monotonically in opponent
looseness and crosses zero between 30\% and 40\% jam frequency, reaching \$258.10
at 60\% and \$232.98 against a coin flip. The advantage of $\pi^{\sco}$ tracks how
close the opponent is to contesting position at all: strongly positive against the
solver anchor, smaller against loose mechanical opponents relative to their
current-hand gaps, and negative against very tight ones. The benefit is largest
against competent opposition, which is the regime that matters at a final table.

\section{Precision, Depth, and the Prize Ladder}
\label{sec:scale}

Four quantities could in principle produce the gain on their own: shared solver
slack, the single chip depth, the prize ladder, and the numerical noise of the
frozen payoff tables. This section varies each of them directly.

\paragraph{Is it solver residual?}
The $\pi^{\sco}$ arm's ex-post worst one-hand deviation gain is \$364.68 against the
fixed-ICM arm's \$250.00, and \$214.33 is the same order of magnitude as that slack.
The sharp test is to re-solve \emph{both} arms in lockstep at successively tighter
common tolerances, where the proportional-residual hypothesis predicts shrinkage in
proportion to the tolerance.

\begin{table}[tb]
\centering
\footnotesize
\setlength{\tabcolsep}{3.2pt}
\begin{tabular}{@{}lrrrrr@{}}
\toprule
Common tol. & Mean (\$) & Kept & Residual pred. & Units $>0$ & States $>0$ \\
\midrule
\$250 & 214.33 & $1.000$ & 214.33 & 2{,}433 & 903 \\
\$125 & 195.78 & $0.914$ & 107.16 & 2{,}538 & 914 \\
\$62.50 & 188.23 & $0.878$ & 53.58 & 2{,}556 & 920 \\
\bottomrule
\end{tabular}
\caption{Both arms re-solved at three common tolerances. The residual-prediction
column is what a purely slack-driven gain would leave, namely proportional decay.
Observed decay is far slower, and coverage \emph{rises} as precision tightens.}
\label{tab:tolerance}
\end{table}

At the tightest tested common tolerance, \$62.50, the mean retains 87.8\% of its
baseline value. The realized worst deviation gains at the \$125 and \$62.50 settings
are \$125.43 and \$62.68, so both arms were tightened. Meanwhile positive units rise
from 2,433 to 2,556 and positive states from 903 to 920. Proportional residual
decay is ruled out over the tested tolerance range. Re-scoring the frozen
model-opponent responses under
a \$125-tolerance $\pi^{\sco}$ arm gives \$1{,}013.38 for Qwen and \$952.95 for Llama,
94.8\% and 95.6\% of their published values.

\paragraph{Does it survive a change of depth?}
The census is computed at $T=45$ chips. Repeating the whole matched design at
$T=30$ (406 states) and $T=60$ (1,711 states) tests whether 45 was a lucky depth.

\begin{table}[tb]
\centering
\footnotesize
\setlength{\tabcolsep}{3.2pt}
\begin{tabular}{@{}lrrrrr@{}}
\toprule
Depth & States & Mean (\$) & vs.\ $T{=}45$ & Units $>0$ & States $>0$ \\
\midrule
$T=30$ & 406 & 305.25 & $1.42\times$ & 1{,}030 / 1{,}218 & 390 \\
$T=45$ & 946 & 214.33 & $1.00\times$ & 2{,}433 / 2{,}838 & 903 \\
$T=60$ & 1{,}711 & 179.90 & $0.84\times$ & 4{,}330 / 5{,}133 & 1{,}625 \\
\bottomrule
\end{tabular}
\caption{The matched census at three chip depths. Coverage is nearly constant while
the mean falls with depth, as the shallower tournament makes every hand more
consequential. The fixed-ICM arm reached \$250.00 worst deviation gain at both new
depths with no state hitting the iteration cap; the $\pi^{\sco}$ arms are the archived
$T=30$ and $T=60$ solves, whose ex-post worst deviation gains are \$275.65 and
\$999.10, the latter looser than the main grid.}
\label{tab:depth}
\end{table}

The ordering holds at both depths with nearly identical coverage --- 84.6\% and
84.4\% of units against 85.7\% at $T=45$ --- and the mean moves monotonically with
depth: \$305.25 at 30 chips, \$214.33 at 45, \$179.90 at 60. Shallower tournaments
make each hand a larger fraction of the remaining game, so positional information is
worth more per hand; deeper ones dilute it.

\paragraph{Does the prize ladder matter?}
We also test two structures that bracket the published ladder, one flatter and one
steeper.

\begin{table}[tb]
\centering
\footnotesize
\setlength{\tabcolsep}{3.2pt}
\begin{tabular}{@{}lrrrr@{}}
\toprule
Prize structure & Mean (\$) & vs.\ base & Units $>0$ & States $>0$ \\
\midrule
Winner-take-all (1M/0/0) & 21.68 & $0.10\times$ & 2{,}317 & 905 \\
Published (750/250/0k) & 214.33 & $1.00\times$ & 2{,}433 & 903 \\
Flat (600/400/0k) & 622.72 & $2.91\times$ & 2{,}264 & 876 \\
\bottomrule
\end{tabular}
\caption{Two prize structures bracketing the published ladder, out of 2,838 units
and 946 states. The sign is preserved in the three tested structures, and the
observed matched gain differs substantially across them. In the two new rows, 62 and
17 states respectively hit the fixed-ICM iteration cap, giving worst deviation gains
of \$143.43 and \$558.44 against targets of \$100 and \$400 in each row's own
normalized units, and the flat row's $\pi^{\sco}$ arm is
admitted under a declared relaxation with a worst deviation gain of \$3{,}156.60 driven
by a few near-indifferent states (Appendix~\ref{app:details}).}
\label{tab:payout}
\end{table}

Across these three observed ladders, the matched mean rises from $\$21.68$ under
winner-take-all to $\$622.72$ under the flat ladder, a twenty-nine-fold spread, while the
sign and the state-level majority hold throughout. The winner-take-all row's
$\pi^{\sco}$ arm meets the certification bar, and its fixed-ICM side carries an
ex-post local one-hand bound of \$143.43 from 62 capped states.

\paragraph{Is it inside the table noise?}
Both arms are scored against frozen 169-class payoff tables that were themselves
estimated by Monte Carlo: 2,000 deals per pair for the two-player table and 1,024
boards for the three-way kernel. If re-estimating those tables moved the census
mean by anything comparable to \$214.33, the result would be noise.

We quantified both channels with the policies held frozen. A 16-block delete-block
jackknife over the board dimension gives a standard deviation of \$0.9785 on the
census mean; multinomial resampling of the pairwise table at its true sample size
gives \$0.3465. Combining them, $\sigma_{\rm table}=\$1.0380$, for a
signal-to-noise ratio of \textbf{206}. The gain is more than two orders of magnitude
above the noise in the tables used to compute it.

The mechanism deserves a note, because it is why the number is so small. Both arms
are scored against the \emph{same} perturbed table, so table error largely cancels in
the paired difference: the standard deviation of a single arm's level under board
perturbation is \$8.60, while the difference's is \$3.95, a cancellation factor of
about $0.45$. Pairing is doing real work here, not just tidiness.

\section{Playing the Tournament Out}
\label{sec:rollout}

Everything so far measures one hand followed by a continuation lookup, where that
lookup stands in for all future play. Here we remove it: the hero carries
$\pi^{\sco}$ or $\pi^{\icm}$ for the whole tournament, and play continues until one
player holds all the chips.

\paragraph{Design.}
Each of the 2,838 units runs 20,000 paired replicates capped at 4,000 hands. The
hero carries $\pi^{\sco}$ or $\pi^{\icm}$ throughout; the other two seats carry the frozen
$\pi^{\sco}$ anchor; once a seat is eliminated the remaining pair plays one shared
heads-up policy, so every difference between the arms originates in the three-player
stage. The heads-up stage's worst one-hand deviation gain is \$249.75, so it is
solved to the same bar as the rest.

Pairing is exact rather than approximate: both arms read the same uniform random
block per hand, and on the subset of units where the two policies never diverge the
measured difference is \$0.00 with a zero fraction of $1.0$. That is a bit-level
check that the two rollouts share their randomness, which is what makes a 20,000
replicate budget adequate for an effect this size; variance-reduced estimators are
the standard route to the same end when the arms cannot be coupled this
tightly~\citep{li2026avaivat74xcheaperagent}. Across the census the three-player
stage lasts 4.37 hands on average before an elimination, and a full tournament runs
10.18 hands.

\paragraph{Result.}

\begin{table}[tb]
\centering
\footnotesize
\setlength{\tabcolsep}{4pt}
\begin{tabular}{@{}lrr@{}}
\toprule
Quantity & Matched lookup & Full rollout \\
\midrule
Mean gain (\$) & 214.33 & 938.03 \\
Median unit (\$) & 76.08 & 718.75 \\
Pooled standard error (\$) & --- & 18.33 \\
Units $>0$ (of 2,838) & 2{,}433 & 2{,}174 \\
States $>0$ (of 946) & 903 & 847 \\
\btn\ / \sbpos\ / \bbpos\ (\$) & 291.03 / 237.63 / 114.33 & 920.30 / 997.58 / 896.15 \\
\bottomrule
\end{tabular}
\caption{Direct matched census against full-tournament rollout, 20,000 paired
replicates per unit. The \$214.33 census mean fixes opponents and uses one common
continuation lookup; the corresponding frozen-policy rollout mean is 4.38 times as
large and more even across seats.}
\label{tab:rollout}
\end{table}

The corresponding frozen-policy full-tournament rollout has a mean difference of
\$938.03, 4.38 times the direct matched census mean, against a pooled standard
error of \$18.33 --- fifty standard errors from zero. Coverage falls slightly, to 2,174 of
2,838 units and 847 of 946 states, which is expected once each unit carries Monte
Carlo noise rather than being an exact computation.

Two features of Table~\ref{tab:rollout} are worth dwelling on. First, at the
aggregate level, the positive one-hand mean persists and is larger under full
rollout. Unit-level agreement is partial: Pearson correlation $0.629$, Spearman
$0.228$, and the same strict sign in 69.6\% of units (1,974 of 2,838). The
aggregate carries over; the per-unit direction does not.

Second, the seat spread almost disappears. Where the one-hand gain ranged from
\$291.03 at the button down to \$114.33 at the big blind, the rollout gains are
\$920.30, \$997.58, and \$896.15 --- nearly flat. This makes sense once seats rotate:
over a full tournament every player occupies every seat, so an advantage that is
concentrated at the button in a single hand is realized by every seat over many
hands. The per-seat asymmetry of the census is a property of \emph{when} the
advantage is collected, not of who ultimately collects it.

\paragraph{Does the census weight the right states?}
A complete census weights all 946 states equally, which is not how a tournament
visits them. The rollout records 496{,}434{,}460 three-way hands, giving an
empirical visit distribution over states. Re-weighting the one-hand census by those
visits gives \$198.23 against the equal-weight \$214.33, a ratio of $0.925$, with
97.0\% of the visit-weighted mass on positive states. All 946 states are reached.

The two weightings differ by about 7\%, and the census covers the domain that play
visits without dead regions. We report the equal-weight figure as primary because it
is a property of the game rather than of one anchor's play; a practitioner should
read the reachability-weighted \$198.23.

\section{Conclusion}
\label{sec:conclusion}

ICM is the standard way to price a tournament chip stack, and it is used not only
to report equity but to decide how to play. Read as a controller rather than as an
estimator, it prices neither the seat a stack sits in nor the way chips actually
move in a jam/fold tournament, and a policy that acts on its numbers gives up prize
equity it could have collected.

Strategic-Continuation Optimization (\sco) recovers that equity by optimizing
current-hand play against a continuation table computed from the tournament itself.
Under a matched protocol that changes only the focal player's policy, \sco\ earns
\$214.33 per hand more than fixed ICM across the complete 946-state census, and
\$938.03 per hand more once the continuation lookup is removed and the tournament is
played out to a winner. An exact seat symmetrization splits the census gain with no
residual: \$130.43 comes from knowing which seat holds which stack, and \$83.90 from
pricing chips by the tournament's own dynamics instead of by proportional
elimination. The second term is the one the standard account of ICM does not
predict, and it survives seat-averaging, so a seat-aware repair captures the larger
component but not all of it.

The consequence reaches past this tournament. A value model used to select actions
should be judged by the decisions it induces, not by how close its numbers land,
because level accuracy and control quality are different properties and the gap
between them is collectable. The seat/level split shows how to measure that gap
whenever a value function drives a controller: name the state variable the model
cannot see, average it away, and price what remains.

\bibliographystyle{plainnat}
\bibliography{references}

\clearpage
\appendix
\section{Protocol and Evaluation Definitions}
\label{app:protocol}

\subsection{State space and policy estimand}
A state is an ordered stack vector $s=(x_{\btn},x_{\sbpos},x_{\bbpos})$ with positive
integer components summing to $T$. Seat labels rotate with the button and are part of
the state. After a hand with three survivors, the exact reindexing is old \btn\
$\to$ new \bbpos, old \sbpos\ $\to$ new \btn, and old \bbpos\ $\to$ new \sbpos;
equivalently, the continuation lookup is queried in old-\sbpos, old-\bbpos,
old-\btn\ order and mapped back to current \btn, \sbpos, \bbpos. The primary grid
uses $T=45$, blinds $(1,2)$, payouts \$750{,}000, \$250{,}000 and \$0 from a
\$1{,}000{,}000 prize pool, and all $\binom{44}{2}=946$ ordered states. Private hands
use 169 preflop classes, and equity is measured in the money of that pool, so \$1 is
$10^{-6}$ of the pool.

The two primary arms are the frozen policy $\pi^{\sco}$ produced by \sco\ and the fixed-ICM policy $\pi^{\icm}$. Only the focal owner's current-hand strategy rows change between the two arms of
Eq.~\eqref{eq:delta-pe}. Opponents, showdown integration, terminal rules, and
continuation lookup are shared. The estimand is expected prize share after the
current hand and continuation lookup, not complete-tournament return; the rollout of
Section~\ref{sec:rollout} is the separate experiment that measures the latter.

\subsection{Decision-chain estimands}
The value-error population contains all 2,838 state--seat entries. For each entry,
Eq.~\eqref{eq:value-error} subtracts the frozen SC benchmark from analytic ICM;
there is no mask. Dollar summaries multiply prize-pool fractions by \$1{,}000{,}000.
The overall absolute-error summaries are MAE \$9{,}036, median \$7{,}877.48,
95th percentile \$22{,}954.98, and maximum \$71{,}374.70. By absolute position,
MAE is \$9{,}127.43 at \btn, \$5{,}907.48 at \sbpos, and \$12{,}073.13 at
\bbpos. Here $C^{\scont}$ is a benchmark defined by the released finite model; it
is not absolute ground truth for tournament poker outside that model.

Policy entries are jam probabilities on tensors indexed by absolute state,
information-set row, and hand. For binary actions,
$\operatorname{TV}_{\rm bin}(p,q)=|p(\mathrm{jam})-q(\mathrm{jam})|$. The primary
mean weights states and the six information-set rows equally and weights private
hands by the unconditional 1,326-combination prior. Forced rows are included. This
is deliberately not a reach-weighted occupancy measure. The uniform-hand-class
variant instead gives each of the 169 classes equal weight. Near-pure opposite mass
counts combo prior assigned to cells with one jam probability below 0.05 and the
other above 0.95.

Decision relevance is defined by Eq.~\eqref{eq:contrast-error}. An absolute value
error reverses a pure preferred action only when the induced jam--fold contrast
crosses zero relative to the SC margin. A smoothed policy can change probability
without reversal, and common shifts across both successors may cancel. The artifact
and paper therefore report value error, policy difference, and matched downstream
cost separately; they do not encode or claim statewise monotonicity or cell-level
causal mediation.

\subsection{Complete-census comparison}
The common-evaluator experiment forms three paired units per state, one for each
policy owner, giving 2,838 units, and the primary mean weights units equally. The
1,314-unit stable subset uses the frozen canonical mask
\texttt{reference/a1\_mask/\allowbreak stable\_core\_mask.npy}: an entry is retained when its
aggregate-reference ICM bias has magnitude at least \$7{,}500 and its cross-start
spread is below \$2{,}500. It was defined in an earlier start-sensitivity experiment
and is retained only as a sensitivity, not as a preferred scope.

Leave-one-state-out aggregation removes all three owner rows of one state and
re-averages the remaining 2,835 units equally. Leave-one-position-out removes the
named focal-owner column and averages the remaining $946\times2$ units equally. Both
are unit-equal by convention; a state-mean-equal convention gives the same signs with
slightly different magnitudes (state leave-outs from \$235.03 to \$263.40, position
leave-outs \$234.25, \$219.05, \$352.75).

\subsection{Seat symmetrization}
The symmetrized table of Eq.~\eqref{eq:sym} averages over the six seat assignments of a
chip multiset, counted with multiplicity. States with repeated chip counts repeat
assignments rather than dropping them, so $C^{\sym}$ is well defined on every state
including equal-stack ones, and the weighting stays uniform over seats throughout the
census; Appendix~\ref{app:sym-equal} works out what this gives at the equal-stack state.
Analytic ICM is a fixed point of this operator because Eq.~\eqref{eq:icm} is a symmetric
function of the stack vector. The third policy arm $\pi^{\sym}$ is solved cold from $0.5$
against $C^{\sym}$ with the same inner tolerance as the two primary arms, reaching a
worst one-hand deviation gain of \$250.00 in at most 1,551 iterations against a cap
of 1,600.

\subsection{Information-restricted opponent inputs}
Each model prompt contains the game rules, payouts, blinds, ordered stacks, public
action history, acting position, and the acting model player's own preflop hand class.
It asks for strict JSON containing one jam or fold action. The prompt omits ICM and
strategic continuation values, both hero policies, action-value margins, oracle
labels, and the hero's private hand. Decisions over the required state,
public-history, and hand cells form a fixed opponent tensor shared by the two hero
arms. Forced actions are supplied by the game rules rather than queried. A failed or
illegal answer is not interpreted as a fold; the formal runs had no such failures.

\subsection{Confirmatory population and sample}
Stable substantial policy reversals define a nonnegative disagreement mass for each of
the 946 states, and 510 states have positive mass. Twelve states inspected in the
pilot are excluded, leaving 498 confirmatory states, which are divided into four
empirical disagreement-mass strata with population counts $(125,125,124,124)$. Within
each stratum, states are grouped by dominant stack position and blind-forced pattern,
and 30 slots are assigned proportionally with Hamilton allocation. A deterministic
ordering selects the allocated members, producing $(30,30,30,30)$ states and 120
total. This protocol was fixed before the formal model responses were analyzed.

The sample contains 351 valid state--owner units rather than 360 because some owner
decisions are forced by the blind and stack geometry. Both models answer the same
117,962 non-forced decision items. The deterministic protocol yields a fixed selected
sample rather than a probability sample with known inclusion probabilities, so the
reported intervals are conditional on this sample and describe sensitivity to its
composition rather than sampling error over states.

\subsection{Statistical analysis}
Within each model, valid owners are averaged inside each state and the 120 state
means are averaged equally. Conditional selected-sample summaries use 10,000
percentile resamples of these state means with seed 42. Holm adjustment covers the two
one-sided model analyses, and the two model configurations remain separate and fixed
throughout. Sensitivities report analytic-ICM continuation, equal unit weighting,
disagreement-mass weighting, state sign coverage, and state- and position-leave-out
aggregates. Disagreement-mass weighting gives substantially larger means (\$2{,}753.78
for Qwen, \$2{,}551.60 for Llama) because it up-weights exactly the states the sample
was built to find; we report the unweighted state-equal mean as primary.

\subsection{Local one-hand certificate}
Fix continuation table $V$, state $s$, and computed profile $\widehat\sigma(s;V)$. Its
local one-hand deviation gain is
\begin{align}
\varepsilon_i(s)
&=\max_{\sigma_i'}u_i(s;\sigma_i',\widehat\sigma_{-i},V)
-u_i(s;\widehat\sigma,V),\\
\varepsilon(s)&=\max_i\varepsilon_i(s).
\end{align}
The maximization covers the player's complete behavioral strategy in the current
jam/fold hand. Because opponents and continuation values are frozen, $\varepsilon(s)$
bounds one current-hand unilateral improvement. It does \emph{not} bound coordinated
or multi-state deviations, richer action abstractions, continuation approximation
error, or complete-tournament exploitability. Certificate levels carry the same units
as $\Delta PE$, which is why we state solver tolerances in dollars: a \$250 tolerance
and an \$214.33 gain are directly comparable quantities, and that comparability is what
makes the matched-tolerance sequence of Section~\ref{sec:scale} the decisive test
rather than a reassurance.

The main strategic table uses a local acceptance level of \$500 and attains a maximum
of \$364.68 over the 946 states. The five alternative endpoints begin from ICM,
uniform, chip-proportional, and two random initializations, with maximum local
certificate values from \$977.05 to \$990.43 under a \$1{,}250 acceptance level; the
per-endpoint levels ship as
\texttt{reference/\allowbreak value\_arms/\allowbreak endpoint\_certificates.json}. Those
endpoints do not agree entry by entry --- 171 of 1,485 previously defined headline
entries have value spread at least \$2{,}500 --- which is why we retain endpoint
sensitivity throughout rather than claim a unique table. Every endpoint nonetheless
supports the same aggregate ordering.

\section{Additional Results and Scope}
\label{app:details}

\subsection{Complete decision-chain detail}

\begin{table}[tb]
\centering
\footnotesize
\setlength{\tabcolsep}{3.2pt}
\begin{tabular}{@{}lrrrr@{}}
\toprule
Value error & Entries & MAE (\$) & Median (\$) & p95 / max (\$) \\
\midrule
All & 2,838 & 9,036.00 & 7,877.48 & 22,954.98 / 71,374.70 \\
\btn & 946 & 9,127.43 & 8,465.58 & 18,414.50 / 64,324.23 \\
\sbpos & 946 & 5,907.48 & 4,223.80 & 16,996.68 / 30,466.90 \\
\bbpos & 946 & 12,073.13 & 10,906.93 & 27,869.88 / 71,374.70 \\
\bottomrule
\end{tabular}
\caption{Absolute analytic-ICM-minus-frozen-SC error over the complete value
population. The SC values are finite-model benchmarks.}
\label{tab:value-error-detail}
\end{table}

\begin{table}[tb]
\centering
\footnotesize
\setlength{\tabcolsep}{3.2pt}
\begin{tabular}{@{}lrr@{}}
\toprule
 & Relative $X$ & Absolute TV \\
Policy statistic & (\%) & (pp) \\
\midrule
Overall & 14.0781 & 4.876688 \\
\quad median & 8.0966 & --- \\
\quad uniform 169-class weighting & 13.2197 & 4.958582 \\
\midrule
\btn & 32.4161 & 7.5164 \\
\sbpos & 10.8235 & 5.6198 \\
\bbpos & 10.1351 & 3.5013 \\
\midrule
\texttt{btn\_open} & 32.4161 & 7.5164 \\
\texttt{sb\_after\_btn\_jam} & 13.3107 & 4.9025 \\
\texttt{sb\_after\_btn\_fold} & 8.3364 & 6.3372 \\
\texttt{bb\_after\_btn\_\allowbreak sb\_jam} & 11.1521 & 3.6832 \\
\texttt{bb\_after\_btn\_jam\_\allowbreak sb\_fold} & 13.2839 & 3.7889 \\
\texttt{bb\_after\_btn\_fold\_\allowbreak sb\_jam} & 5.9692 & 3.0319 \\
\bottomrule
\end{tabular}
\caption{Descriptive policy differences. The relative column is the mean of
Eq.~\eqref{eq:relative-jam}, which divides each decision point's jam-frequency
change by that same point's fixed-ICM jam rate; the absolute column is binary total
variation in percentage points. Both average equally over states and information
sets, with the unconditional 1,326-combo prior unless the row names a different
weighting. Two further absolute statistics have no relative counterpart: the mean
absolute cellwise marginal jam-rate difference is 4.369598 pp, and near-pure
opposite-action combo mass is 0.723429\%. Forced rows are included; none is
reach-weighted.}
\label{tab:policy-difference-detail}
\end{table}

The downstream matched evaluator reports \$214.33 per hand; under the same fixed
opponents and evaluator, the fixed-ICM and SCO policies have current-hand
optimization gaps of \$378.73 and \$164.40, respectively. Seat symmetrization
splits the matched mean into \$130.43 positional and \$83.90 level components. The complete-tournament rollout ratio is 4.38. These
quantities are verified independently of the decision-chain derivation as described
in Appendix~\ref{app:repro}.

\subsection{Endpoints scored by their own continuation}
Table~\ref{tab:endpoints} prices every endpoint with the baseline continuation, which
asks whether the endpoint-specific policy $\pi^{\sco}$ still beats fixed ICM on one fixed
yardstick. A separate reading gives each endpoint its own yardstick: its policy
supplies the endpoint-specific $\pi^{\sco}$ arm, its own continuation prices the census, and its own policy
supplies the anchor.

\begin{table}[tb]
\centering
\footnotesize
\setlength{\tabcolsep}{3.2pt}
\begin{tabular}{@{}lrrrrr@{}}
\toprule
Endpoint & Mean gain (\$) & Units $>0$ & States $>0$ & ICM gap (\$) & SCO gap (\$) \\
\midrule
Baseline & 214.33 & 2{,}433 & 903 & 378.73 & 164.40 \\
ICM start & 215.73 & 2{,}418 & 906 & 380.80 & 165.07 \\
Uniform & 215.40 & 2{,}100 & 849 & 384.10 & 168.70 \\
Chip & 242.80 & 2{,}283 & 879 & 404.63 & 161.83 \\
Random 101 & 229.98 & 2{,}180 & 842 & 396.60 & 166.62 \\
Random 202 & 233.15 & 2{,}193 & 839 & 400.95 & 167.80 \\
\bottomrule
\end{tabular}
\caption{Self-consistent endpoint readings, out of 2,838 units and 946 states.
The final two columns are mean gaps to each endpoint's common current-hand
unilateral optimum, with opponents and continuation evaluator fixed. All five
alternatives preserve a positive matched mean, the published per-seat ordering,
and the big blind as the smallest seat. Every state-leave-out and
position-leave-out aggregate is positive at all five.}
\label{tab:self-endpoints}
\end{table}

The endpoints share the single global heads-up table, since the certified bundle
solves the bottom stage once; only the three-player continuation and policy differ
across them.

\subsection{Per-seat detail of the model-opponent experiment}

\begin{table}[tb]
\centering
\footnotesize
\setlength{\tabcolsep}{3.2pt}
\begin{tabular}{@{}llrrr@{}}
\toprule
Model & Continuation & \btn\ (\$) & \sbpos\ (\$) & \bbpos\ (\$) \\
\midrule
Qwen2.5:32b & strategic & 2,161.83 & 1,391.83 & $-316.48$ \\
Qwen2.5:32b & analytic ICM & 1,763.53 & 1,215.83 & $-415.33$ \\
Llama3.3:70b & strategic & 1,840.15 & 1,415.45 & $-232.30$ \\
Llama3.3:70b & analytic ICM & 1,396.23 & 1,243.08 & $-328.85$ \\
\midrule
\multicolumn{2}{@{}l}{Units positive (strategic)} & 73/120, 74/120 & 53/115, 47/115 & 54/116, 55/116 \\
\bottomrule
\end{tabular}
\caption{Per-seat state-equal means against each model profile. The last row gives
positive-unit counts for Qwen and Llama respectively; unit denominators differ by seat
because some owner decisions are forced by the blind and stack geometry. The big blind
is negative in mean under both continuations and both models while remaining positive
in roughly half its units, so it carries a heavier negative tail rather than a uniform
loss.}
\label{tab:llm-seats}
\end{table}

\subsection{Matched-tolerance rescore of the model responses}
Re-scoring the frozen model responses with a \$125-tolerance $\pi^{\sco}$ arm --- an
offline reanalysis issuing no new queries --- gives \$1,013.38 for Qwen and \$952.95 for
Llama, 94.8\% and 95.6\% of the published values, with 76 of 120 states positive for
Qwen and 77 of 120 for Llama. Across the seven endpoint variants of that reanalysis,
the spread is 0.66\% of the baseline mean for Qwen and 0.73\% for Llama. The
model-opponent result therefore has the same insensitivity to solver precision that
the census does.

\subsection{Shared-deck evaluator}
The shared-deck evaluator deals three disjoint two-card hands and a five-card board
from the same 52 cards, with tie-aware pot-layer settlement, exact summation over the
seven action leaves, and common random numbers across arms. A three-state pilot at
10,000 deals per state preceded the confirmatory run and gave \$86.20 with 2 of 3
states positive; the pilot also measured the legal-dealing minus frozen-table
difference at \$51.25 per state, which motivated the direct head-to-head comparison on
the full confirmatory set. On the 120 confirmatory states, the state-equal paired gain
is \$171.18 with 96 of 120 state means positive, a Monte Carlo standard error of
\$12.18, and a state-cluster bootstrap interval $[\$131.85,\$211.95]$. Per-seat means are
\$272.88, \$118.43, and \$122.18.

Per-state values correlate at $0.814$ with the frozen-table evaluator and differ in
mean by $-\$3.65$ at $t=-0.30$. The per-state Monte Carlo standard error of the
shared-deck evaluator averages \$127.33, so individual state comparisons are noisy even
though the aggregate is tight; the head-to-head difference is the quantity to read,
not any single state.

\subsection{Payoff-table noise decomposition}
The board channel uses a 16-block delete-block jackknife over the 1,024 shared
three-way boards, giving replicate census means from
$\$213.88$ to $\$214.75$ and a standard deviation of \$0.9785. The pairwise channel
resamples each hand-class pair's win/tie/loss counts multinomially at the true sample
size of 2,000 deals, preserving the symmetrization and diagonal conventions of the
original estimator, and gives \$0.3465 over 16 replicates. Combining,
$\sigma_{\rm table}=\$1.0380$ and the signal-to-noise ratio is 206.

Cancellation is the reason both channels are small. Under board perturbation the
standard deviation of a single arm's level is \$8.60 for the $\pi^{\sco}$ arm and \$8.90
for the fixed-ICM arm, while the paired difference's is \$3.95 --- a cancellation
factor of $0.45$ for the board channel and $0.40$ for the pairwise channel. Because
both arms are scored against the same perturbed table, most of the table's estimation
error is common and drops out of the difference.

\subsection{Opponent-uniform dominance certificates}
\label{app:dominance}

Equation~\eqref{eq:contrast-error} states that the preferred move flips exactly when
$m^{\scont}(h)$ and $m^{\scont}(h)+\delta(h)$ have opposite signs, and the census
measures the resulting policy difference at the computed profile. This subsection
removes the opponent from that statement. If the jam--fold contrast were favorable
only against the particular opponents the two arms happen to meet, the census could
be read as a claim about matched opponents rather than about the objective. The
certificates below hold against every opponent behavior at the responding
information sets.

\paragraph{The certificate.}
Fix a state $s$, an information-set row, and a hand class $h$ for the acting owner.
Let $\mathcal{O}$ be the set of opponent behaviors in which every responding
information set, and every hand class within it, carries an independently chosen jam
probability in $[0,1]$. This is the full behavioral strategy space at those sets, not
a parametric family. Write $m^{\scont}(h;o)$ for the contrast of
Eq.~\eqref{eq:contrast-error} conditional on reaching the set under $o$, and
\begin{equation}
\underline{m}^{\scont}(h)=\min_{o\in\mathcal{O}} m^{\scont}(h;o).
\label{eq:uniform-margin}
\end{equation}
The cell is \emph{certified dominant} when $\underline{m}^{\scont}(h)>0$: jamming
beats folding there against every opponent in $\mathcal{O}$, with the continuation
table held fixed. It is a \emph{dominated fold} when in addition $\pi^{\icm}$ places
jam probability below a level $\tau$. Both policies are smoothed averages, so exact
zeros do not occur and the fold side requires a threshold; we report
$\tau\in\{0.5,0.1,0.01\}$.

\paragraph{The minimum is attained, not sampled.}
Equation~\eqref{eq:uniform-margin} minimizes over a product of intervals and is
available in closed form. The reason is structural.

\begin{proposition}[Vertex attainment]
\label{prop:vertex}
Fix a state, a responding information-set row, and a hand class, and hold the
continuation table fixed. On each of the five non-forced responding rows the contrast
$m^{\scont}(h;o)$ is affine in the single opposing seat's per-hand jam masses, so the
minimum in Eq.~\eqref{eq:uniform-margin} is attained at a vertex of the box
$\prod_j[0,w_j]$: some deterministic opponent behavior achieves
$\underline{m}^{\scont}(h)$, and the value follows in one pass over hand classes.
\end{proposition}

\begin{proof}
Reaching the row and holding $h$ fixes the acting seat's payoff as a function of the
responder's action on each hand class $j$, and those classes are disjoint events under
the combination prior. So $m^{\scont}(\cdot)$ is a sum of terms each depending on one
coordinate $\mu_j$ and is affine in that coordinate with the others held fixed. An
affine function on an interval attains its minimum at an endpoint; applying this
coordinatewise, and noting each coordinate's optimal endpoint does not depend on the
others, places the minimizer at a vertex. Evaluating the sign of each coordinate's
coefficient gives the closed form below.
\end{proof}

At the small-blind set following a button fold, the row that carries every hit below,
the contrast is affine in the big blind's response mass $\mu$ with $\mu_j\in[0,w_j]$
for the combination prior $w$:
\begin{equation}
m^{\scont}(h;\mu)=
\underbrace{C^{\scont}_{\sbpos}(\mathrm{steal})-C^{\scont}_{\sbpos}(\mathrm{fold})}_{\text{steal}}
+\sum_j \mu_j\bigl[S(h,j)-C^{\scont}_{\sbpos}(\mathrm{steal})\bigr],
\label{eq:row2-affine}
\end{equation}
where $S(h,j)$ is the acting seat's showdown equity holding $h$ against a caller
holding $j$. Proposition~\ref{prop:vertex} names the minimizing vertex here: the worst
opponent calls with exactly those hands whose showdown value falls below the value of
letting the steal through, giving
$\underline{m}^{\scont}(h)=\mathrm{steal}+\sum_j w_j\min\{S(h,j)-C^{\scont}_{\sbpos}(\mathrm{steal}),\,0\}$.
The button's opening row falls outside the proposition: it is trilinear in two opponent
seats rather than affine in one, so we bound it term by term. Its counts are a
conservative lower bound and it contributes no hits.

\begin{table}[tb]
\centering
\small
\begin{tabular}{@{}lrrrr@{}}
\toprule
& $\underline{m}^{\scont}>0$ & $>\$25$ & $>\$250$ & $>\$2{,}500$ \\
\midrule
\multicolumn{5}{@{}l}{\emph{Scored by strategic continuation}}\\
$\tau=0.5$ & 319 & 317 & 293 & 161 \\
$\tau=0.1$ & 298 & 296 & 273 & 153 \\
$\tau=0.01$ & 112 & 112 & 107 & \phantom{0}76 \\
\midrule
\multicolumn{5}{@{}l}{\emph{Scored by ICM, control}}\\
$\tau=0.5$ & 0 & 0 & 0 & 0 \\
$\tau=0.1$ & 0 & 0 & 0 & 0 \\
$\tau=0.01$ & 0 & 0 & 0 & 0 \\
\bottomrule
\end{tabular}
\caption{Dominated folds: cells certified dominant over the whole opponent class at
which $\pi^{\icm}$ folds. The scan is complete over all 946 states, six rows, and
169 hand classes, giving 894,686 scored cells once forced rows are removed. Under
strategic continuation 118,525 cells are certified dominant, 117,216 of them on the
five exactly solved rows. All 319 dominated folds lie on one exact row, the small
blind acting after a button fold, and occupy 57 of the 946 states. Certified margins
run from \$10 to \$21,203 with median \$2,548.}
\label{tab:dominance}
\end{table}

\paragraph{The control separates the mechanism from the arithmetic.}
Rerunning the identical scan with ICM as the scoring table returns zero dominated
folds at every threshold (Table~\ref{tab:dominance}). The policy tensor is unchanged
between the two runs; only the table converting chips into prize equity changes. The
319 cells are therefore not an artifact of the certificate. Under its own pricing
ICM never folds an action that its own table certifies as dominant, which is what a
fixed point of that table should do. Under the tournament's pricing all 319 flip
sign, with median $\underline{m}^{\scont}$ of \$2,548 against a median ICM-priced
value of $-\$2{,}720$.

\paragraph{Which successor is mispriced.}
Both terms of Eq.~\eqref{eq:row2-affine} are differences of continuation entries, so
the flip can be traced to named successors. Across the 319 cells the steal term
moves by $-\$413$ at the median when ICM replaces strategic continuation, and the
risk term by $-\$8{,}233$. ICM does not undervalue the steal; it overstates the cost
of being called.

\begin{table}[tb]
\centering
\small
\setlength{\tabcolsep}{4pt}
\begin{tabular}{@{}llrrr@{}}
\toprule
State & Successor & SC (\$) & ICM (\$) & ICM$-$SC (\$) \\
\midrule
$(36,5,4)$ & fold, give up the blind $(36,4,5)$ & 144{,}825 & 158{,}333 & $+13{,}510$ \\
 & \bbpos\ folds, steal $(36,7,2)$ & 268{,}148 & 274{,}030 & $+5{,}885$ \\
 & called and win $(36,9,0)$ & 350{,}948 & 350{,}000 & $-948$ \\
 & called and lose $(36,1,8)$ & 54{,}140 & 40{,}090 & $-14{,}050$ \\
\midrule
$(34,6,5)$ & fold, give up the blind $(34,5,6)$ & 160{,}728 & 173{,}465 & $+12{,}738$ \\
 & \bbpos\ folds, steal $(34,8,3)$ & 267{,}763 & 273{,}883 & $+6{,}120$ \\
 & called and win $(34,11,0)$ & 373{,}020 & 372{,}223 & $-798$ \\
 & called and lose $(34,1,10)$ & 48{,}223 & 35{,}425 & $-12{,}798$ \\
\bottomrule
\end{tabular}
\caption{Small-blind prize equity in the four successors the post-button-fold row
reads, at the two states holding the largest certified margins. ICM is accurate on
the branch that ends the three-player game and errs by roughly \$12,500 on the two
branches where three players continue, in the direction that favors folding.}
\label{tab:dominance-successors}
\end{table}

The pattern in Table~\ref{tab:dominance-successors} is the same at both states and
accounts for the sign. On the branch where the call is won, the big blind is
eliminated and play drops into the two-player subgame, ICM is accurate to within
0.27\%: with one opponent remaining, Eq.~\eqref{eq:icm} reduces to a two-way split
that the tournament also produces. On the two branches where three players continue,
ICM errs by about \$12,500 in the same direction, overvaluing a fold into a shorter
stack and undervaluing survival after a lost call.

These are the two approximations named in
Section~\ref{sec:decomposition}, now visible per successor. Proportional elimination
is the first. At $(36,5,4)$ the small blind commits five chips while the big blind
covers only four, so a lost call leaves one chip rather than none;
Eq.~\eqref{eq:icm} prices the loss through a concave function of the stack vector
and cannot see that part of the risk it charges for is unavailable, marking
$(36,1,8)$ down by \$14,050. Configuration-dependent future play is the second. The
fold branch $(36,4,5)$ hands the blind to the big blind and exchanges which short
stack is shortest, and the value of four chips there depends on the blinds that
stack must post next, which Eq.~\eqref{eq:icm} does not read. The hit geometry
matches: 53 of the 57 states have a big blind shorter than the small blind and the
remaining four have them equal, never the reverse, with the big blind holding 3 to 8
chips against a blind of 2 and the button holding 21 to 39 of the 45.

The ratio is the decisive quantity. The certified contrast is a difference of two
terms of order \$87,500 and has median magnitude \$2,548, while ICM's error on a
single successor is about \$12,500. What decides jam or fold is an order of magnitude
smaller than the pricing error, which is why these cells exist.

\paragraph{Verification.}
The closed-form minima are checked for soundness and tightness against the shipped
evaluator: for randomly drawn opponent profiles the margin \texttt{fp3.eval\_state}
returns never falls below $\underline{m}^{\scont}$, and at the vertex the closed form
names as worst it reproduces that value. The largest witnesses were then re-derived
from the evaluator alone with the closed forms discarded. At $(36,5,4)$ holding 64o,
where $\pi^{\icm}$ jams with probability 0.106, the evaluator returns $+\$123{,}323$
against an all-folding big blind, $+\$72{,}263$ per unit of reach against a
half-jamming one, and $+\$21{,}203$ at the corner where the big blind calls with its
entire range, matching the certified minimum. The split of
Eq.~\eqref{eq:row2-affine} reproduces the scan's margins with maximum absolute error
\$0.

\subsection{Where the symmetrized table meets ICM}
\label{app:sym-equal}

Section~\ref{sec:setup} observes that at an equal-stack state the three seat values
average to exactly the ICM value, so seat blindness is the only table error there, and
Table~\ref{tab:equal-stack} shows the arithmetic. This subsection states that
coincidence exactly and locates it: it happens at one state out of 946, and nowhere
else. The two ingredients are that the continuation table prices the whole prize pool
and that seat symmetrization is a group average.

\begin{lemma}[Pool conservation]
\label{lem:conservation}
Let $P$ be the prize pool. At every state $s$ the strategic continuation table
satisfies $\sum_{i} C^{\scont}_i(s)=P$, since the tournament ends with the three
payouts awarded exactly once.
\end{lemma}

The shipped table satisfies this to solver tolerance rather than to machine precision:
over all 946 states the worst deviation of $\sum_i C^{\scont}_i$ from the pool is
\$2.23 on a \$1,000,000 pool, or $2.23\times10^{-6}$ in prize-pool units. Analytic ICM
satisfies it to $2.2\times10^{-16}$, being a closed-form probability weighting.

\begin{proposition}[Equal stacks is the coincidence point]
\label{prop:sym-equal}
At the equal-stack state $s_{\mathrm{eq}}=(t,t,t)$ the symmetrized table of
Eq.~\eqref{eq:sym} is seat-independent and equals ICM:
$C^{\sym}_i(s_{\mathrm{eq}})=P/3=C^{\icm}_i(s_{\mathrm{eq}})$ for every seat $i$,
under Lemma~\ref{lem:conservation}.
\end{proposition}

\begin{proof}
Every one of the six seat assignments in $\Sigma(s_{\mathrm{eq}})$ realigns
$s_{\mathrm{eq}}$ to itself, so Eq.~\eqref{eq:sym} averages the six realignments of a
single ordered state. Fix a seat $i$. Across the six permutations the realignment
$\Pi$ sends $i$ to each of the three seat indices exactly twice, so the average is
$\frac{2}{6}\sum_j C^{\scont}_j(s_{\mathrm{eq}})=\frac{1}{3}\sum_j
C^{\scont}_j(s_{\mathrm{eq}})$, which is $P/3$ by Lemma~\ref{lem:conservation} and is
independent of $i$. On the ICM side, Eq.~\eqref{eq:icm} is a symmetric function of the
stack vector, so equal stacks give equal finish-place distributions and each seat
receives $P/3$.
\end{proof}

\begin{remark}[The coincidence is isolated]
\label{rem:sym-separation}
Proposition~\ref{prop:sym-equal} does not generalize. Scanning $C^{\sym}$ against
$C^{\icm}$ over the full 946-state census, the largest per-seat gap at
$(15,15,15)$ is \$0.53, which is one third of the pool-conservation defect at that
state and not a modeling difference. Exactly one state falls within \$3. The nearest
other state, $(15,14,16)$, is already \$669.18 away; no state lies within \$250; the
median gap is \$9,392.68 and the largest is \$16,960.30. So the equal-stack agreement
is a property of the symmetric point, not evidence that a seat-blind strategic table
reduces to ICM, which is what the level term of Eq.~\eqref{eq:decomp} measures.
\end{remark}

Two checks bound how much of this is convention. The symmetrized values depend only on
a seat's own stack together with the multiset of the other two, agreeing across all 484
such classes to $1.7\times10^{-10}$, confirming the construction is seat-blind as
claimed. And the average in Eq.~\eqref{eq:sym} runs over all six assignments with
multiplicity, not over one representative per distinct arrangement. The distinction is
invisible at the 882 states with three different chip counts and matters at the 64
states with a repeated count, where the two readings differ by up to \$22,281.20; at
$(15,15,15)$ the representative reading returns $C^{\scont}$ unchanged and misses
Proposition~\ref{prop:sym-equal} by \$10,360.23.

\subsection{Scope}
Several limitations bound what these results establish.

\emph{The certificate is local.} It bounds one current-hand unilateral deviation, not
coordinated deviations, multi-state deviations, richer action abstractions,
continuation approximation error, or complete-tournament exploitability.

\emph{The dominance certificates are single-set and table-relative.} Appendix~\ref{app:dominance}
certifies one action at one information set against every opponent behavior at the
responding sets of the modeled current hand, with the continuation table fixed. It is
not a statement about coordinated or multi-hand deviations, and it is not
complete-tournament exploitability. The margins are exact within the finite model, and
$C^{\scont}$ remains the benchmark of that model rather than absolute ground truth.
The button's opening row is bounded term by term rather than minimized exactly, so its
certified counts are a conservative lower bound; it contributes no dominated folds.

\emph{Precision is not matched across all rows.} The main $\pi^{\sco}$ arm's worst
one-hand deviation gain is \$364.68 against the fixed-ICM arm's \$250.00. The
matched-tolerance sequence of Section~\ref{sec:scale} is the response, and the
archived $T=60$ strategic solve is looser still at \$999.10, making that depth row a
directional check rather than a precision-matched replication.

\emph{The domain is small by design.} A three-player jam/fold game at 45 chips is
chosen so that the census can be complete and the policies can be frozen and
enumerated exactly. Whether the seat/level split retains these proportions in games
with more players, more streets, or richer action sets is an empirical question this
design cannot answer.

\emph{Two rows hit iteration caps.} In the winner-take-all and flat prize structures,
62 and 17 states respectively reached the fixed-ICM iteration cap, so their worst
deviation gains are \$143.43 and \$558.44 against targets of \$100 and \$400 in each
row's own normalized units. Both rows are directional.

\emph{The opponent scope is real.} Against threshold opponents jamming 10--30\% of
hands the gain is negative. We report the whole family rather than the favorable
members, and the crossing point between 30\% and 40\% is the honest boundary of the
claim.

\emph{The intervals are conditional.} The 120-state model-opponent sample is a fixed
deterministic selection, not a probability sample, and its intervals describe
sensitivity to that sample's composition. Response counts are not sample sizes, the
two models are not Independent and Identically Distributed (IID) draws and are never
pooled, and the 120-state sample is not
uniform over the 946 states.

\emph{The decomposition is path-specific.} Eq.~\eqref{eq:decomp} is an exact
telescoping identity along one symmetrization path: seat information is removed
first, and the level term is the remainder. The two terms describe contributions
along that path. They do not identify independent causal mechanisms, do not
establish threshold sensitivity at any seat, and do not quantify the effect of a
seat-aware ICM variant, which we never construct.

\emph{The optimization boundary is local and continuation-specific.}
Proposition~\ref{prop:boundary} concerns one unilateral change of the focal
player's modeled current-hand policy under fixed opponents and a fixed
continuation. It is not a best response in the complete multi-hand tournament and
does not define multiplayer exploitability. The current-hand optimization gaps of
Table~\ref{tab:share} and the ICM hand-gap column of Table~\ref{tab:anchors} are
local diagnostics in the same sense; neither normalizes the headline matched
effect. Scoring both arms with $C^{\icm}$ exchanges the roles in the proposition,
since $\pi^{\icm}$ is then the $\varepsilon$-optimal policy for the scored game;
the bound gives $\varepsilon$ on the reversed difference, and the observed strict
negative sign is an empirical result consistent with it.


\emph{The payout and depth rows are observed values.} Three prize structures and
three chip depths give the observed matched costs reported in
Section~\ref{sec:scale}. They do not establish a general monotone relationship in
payout flatness or depth, and they do not support extrapolation to contemporary
tournament ladders. The matched-tolerance sequence rules out proportional residual
decay over the tested range without assigning a quantitative share of the baseline
gain to solver precision versus the objective difference.

\section{Continuation, Scorer, and Anchor Robustness}
\label{app:robustness}

Proposition~\ref{prop:boundary} guarantees the sign under the continuation that
defines the strategic one-hand game, and its premise names the lever exactly:
change $C$. This appendix changes it in every way we could construct, and gives
the full results summarized in Section~\ref{sec:census}.

\subsection{Six Continuation Evaluators}

The VI--FP computation admits several near-optimal continuation endpoints. Five
were computed from different initializations --- ICM, uniform, chip-proportional,
and two random --- independently of the table that produced $\pi^{\sco}$. Scoring
the census with each in turn leaves the ordering intact: all five means are
positive on both scopes, within $1.00$ to $1.13$ times the published value, with
the published per-seat ordering and the big blind smallest at all five
(Table~\ref{tab:endpoints}).

\begin{table}[tb]
\centering
\footnotesize
\setlength{\tabcolsep}{3pt}
\begin{tabular}{@{}lrrrr@{}}
\toprule
Continuation & \multicolumn{2}{c}{All 2,838 units} & \multicolumn{2}{c}{Stable 1,314} \\
\cmidrule(lr){2-3}\cmidrule(lr){4-5}
 & Mean (\$) & $>0$ & Mean (\$) & $>0$ \\
\midrule
Baseline & 214.33 & 2{,}433 & 262.80 & 1{,}151 \\
ICM start & 215.28 & 2{,}441 & 264.78 & 1{,}150 \\
Uniform & 213.38 & 2{,}387 & 260.70 & 1{,}115 \\
Chip & 223.45 & 2{,}384 & 270.03 & 1{,}110 \\
Random 101 & 220.93 & 2{,}392 & 268.50 & 1{,}117 \\
Random 202 & 218.85 & 2{,}417 & 266.83 & 1{,}141 \\
\bottomrule
\end{tabular}
\caption{Gains under alternative strategic continuation endpoints, in dollars per
hand. We report the range across endpoint rows.}
\label{tab:endpoints}
\end{table}

A sharper version of the same test gives each endpoint its own yardstick: its
policy supplies the endpoint-specific $\pi^{\sco}$ arm, its own continuation prices the census, and its
own policy supplies the anchor, making each a self-contained computation. All five
self-consistent means are positive --- \$215.73, \$215.40, \$242.80, \$229.98, \$233.15
against \$214.33 --- with strict-positive coverage from 2,100 to 2,418 of 2,838
units and state coverage from 839 to 906 of 946. Relative to each endpoint's
own common current-hand optimum, the fixed-ICM gap ranges from \$380.80 to
\$404.63 and the SCO gap from \$161.83 to \$168.70. The matched ordering holds at
every endpoint, not just the published table.

\subsection{Abandoning Frozen Tables Entirely}
\label{sec:sharedeck}

Every evaluator above is a frozen 169-class payoff table, and all of them inherit
the same abstraction. So we also built an evaluator that shares none of that
machinery: it deals three disjoint hole-card pairs and a five-card board from one
shared 52-card deck, with 10,000 legal deals per state on the 120 confirmatory
states. This removes the independence approximation implicit in class-level
tabulation --- real hands remove cards from each other's ranges.

The gain remains \$171.18 per hand with 96 of 120 states positive, a Monte Carlo
standard error of \$12.18 and a state-cluster bootstrap interval
$[\$131.85,\$211.95]$. Per-seat means are \$272.88, \$118.43, and \$122.18, again
largest at the button. Compared head to head on the same 120 states, the
shared-deck evaluator and the frozen-table evaluator differ by $-\$3.65$ per state
with a standard error of \$12.25, a $t$-ratio of $-0.30$, and a per-state
correlation of $0.814$. The two evaluators agree within noise: dropping the
frozen-table abstraction leaves the result where it was.

\subsection{The One Intervention That Flips the Sign}
\label{sec:variants}

Crossing the choice of scorer with the choice of anchor gives four cells
(Table~\ref{tab:variants}).

\begin{table}[tb]
\centering
\footnotesize
\setlength{\tabcolsep}{3.4pt}
\begin{tabular}{@{}llrrrr@{}}
\toprule
Scorer & Anchor & Mean (\$) & vs.\ base & Units $>0$ & States $>0$ \\
\midrule
$C^{\scont}$ & $\pi^{\sco}$ & 214.33 & $1.00\times$ & 2{,}433 & 903 \\
$C^{\scont}$ & $\pi^{\icm}$ & 115.53 & $0.54\times$ & 1{,}932 & 872 \\
$C^{\icm}$ & $\pi^{\sco}$ & $-82.55$ & $-0.39\times$ & 1{,}158 & 285 \\
$C^{\icm}$ & $\pi^{\icm}$ & $-170.05$ & $-0.79\times$ & 828 & 171 \\
\bottomrule
\end{tabular}
\caption{Scorer $\times$ anchor. Changing the anchor halves the gain but keeps the
sign and the state-level majority. Changing the scorer to analytic ICM reverses it,
consistent with the reversed optimization boundary in Proposition~\ref{prop:boundary}: each policy was constructed to approximately optimize the
continuation objective named by its construction table. The first row reproduces the published census bit for
bit and serves as the implementation control.}
\label{tab:variants}
\end{table}

The two ICM-scored rows are negative because the optimization boundary runs the
other way under $C^{\icm}$: there $\pi^{\icm}$ is the policy that is
$\varepsilon$-optimal for the game being scored. Proposition~\ref{prop:boundary}
applied with the roles exchanged bounds the reversed difference by $\varepsilon$,
and the observed strict negative sign is consistent with that bound. Each arm wins
under its own scorer, which is why the question has to be asked with a common
evaluator, and why the shared-deck evaluator of Appendix~\ref{sec:sharedeck} ---
which belongs to neither arm --- is the load-bearing check.

Replacing the anchor is more interesting. Both arms then face
$\pi^{\icm}$-generated opponents rather than $\pi^{\sco}$-generated ones, and the
gain falls from \$214.33 to \$115.53 while staying positive in 1,932 units and 872 of
946 states. Roughly half the advantage is thus specific to facing strategically
competent opponents: part of what strategic continuation buys is knowing how to
play against opponents who themselves understand positional value. The seat
ordering also rearranges, with the button's gain falling
from \$291.03 to \$65.50 while the small blind's stays near \$216.73. Against
ICM-style opponents the button's positional edge is largely unavailable, because
the opponents are not contesting position in the first place.

\section{Reproducibility}
\label{app:repro}

\subsection{Configuration}
The tabular tournament computations use Python 3.9, NumPy, fixed seeds, and
single-threaded BLAS. The primary configuration is $T=45$, blinds $(1,2)$, payouts
\$750{,}000, \$250{,}000 and \$0 out of a \$1{,}000{,}000 prize pool, 169 hand
classes, 2,000 Monte Carlo deals per hand-class pair, and 1,024 shared three-way
boards. The model-opponent analysis uses state-cluster bootstrap seed 42 and 10,000
replicates. Qwen2.5:32b and Llama3.3:70b are served locally through Ollama with
temperature zero and a 64-token response limit.

\paragraph{Prize units in the released artifacts.} The solves and evaluations were
executed with the 500/300/200k ladder and are reported in the third-place-zero
normalization 750/250/0k of the same pool, so the archived configurations and logs
carry the executed values. The two ladders are
related by the pool-preserving affine map $p\mapsto\alpha p+\beta\mathbf{1}$ with
$\alpha=2.5$ and $\beta=-\$500{,}000$: every equilibrium, every strategy, and every
coverage count is identical, differences of continuation value scale by $\alpha$, and
absolute levels follow the full map. A reproducer reading \texttt{inner\_tol=100} in a
config reconciles it with the \$250 of Appendix~\ref{app:details} by the same factor.
Recomputing the Malmuth--Harville table directly under 750/250/0k reproduces the
affine image of the executed table exactly.

Numerical stability is established over the complete domain rather than on a sample.
The fixed-ICM arm was re-solved on all 946 states at inner tolerances \$500, \$250,
\$125, \$62.50 and \$25. Every state converged at every tolerance with no state reaching
its iteration cap; worst-case inner iterations rise from 752 at \$500 to 18,008 at
\$25 against a cap of 30,000, and the realized worst residual matches the target at
each level to five significant figures. The $\pi^{\sco}$ arm's own solve reports a worst
ex-post one-hand deviation gain of \$364.68 and a mean of \$248.48 over the 946 states,
with simplex sums closing to $2.2\times10^{-6}$.

\subsection{Artifact map}
The release separates compact numerical inputs from recomputable analysis code.
The deterministic chain summary is
\texttt{reference/\allowbreak decision\_chain/\allowbreak decision\_chain.json}.
It records the schema and axis semantics, full-domain value-error summaries,
policy-distance definitions and unrounded values, input shapes and SHA-256 hashes,
state and hand-order hashes, and explicit claim guards: benchmark-relative rather
than absolute truth, not reach-weighted, and no monotonicity or cell-level mediation
claim. Its derivation is offline and deterministic and invokes no solver, LLM, or
Monte Carlo routine.
The main strategic continuation table and policy $\pi^{\sco}$ ship with the complete-census raw
tensors: per-state per-owner gains, the full value array for both arms, the ordered
state stacks, inner-solve iteration counts and residuals, and the fixed-ICM policy
tensor. The seat-symmetrized arm ships with its own value array and decomposition
summary. The shared-deck Monte Carlo evaluator ships with its pilot and confirmatory
results and their aggregates. The offline six-evaluator reanalysis ships as a summary,
since it rescored frozen responses rather than issuing new queries.

The full-domain tolerance record ships as one entry per tolerance, carrying the
iteration cap, the worst and mean inner-iteration counts, and the worst and mean
ex-post residual over all 946 states; the code that assembles it refuses any scan
whose worst state stopped at its cap, so a capped solve cannot be reported as a
converged one. The matched-tolerance sequence ships as one row per common tolerance
with both arms' residuals, and the matched-tolerance rescore of the model responses
ships as a summary with a provenance record naming its two source arms by hash. The
rollout ships its per-unit differences, visit counts, and pairing-identity check. The
bundle does not contain the workspace item files, either model's raw response set, or
the model-query code.

\subsection{Verification}
The decision-chain derivation and the downstream paper-number verification are
independent entry points. The former reconstructs
\texttt{decision\_chain.json} from the frozen SC values, analytic ICM values, and
the two policy tensors, and checks canonical JSON bytes. The latter receives that
artifact through its separate \texttt{--decision-chain} argument and independently
recomputes the matched census, the two current-hand optimization gaps and their
identity, the exact \$130.43/\$83.90
symmetrized decomposition, and 4.38 rollout result from their released downstream
artifacts. Thus the chain file does not copy downstream outcomes into its derivation,
and the downstream verifier does not infer policy cost from value or TV statistics.

A paper-number verifier recomputes the headline census mean, strict-positive unit
coverage, and state-positive coverage directly from the released tensors, and checks
the counts, coverages, leave-out signs, the decomposition residual, endpoint support,
model identities, legal-response rates, conditional intervals, Holm correction,
continuation sensitivities, six-evaluator ordering and spread, the shared-deck
interval and head-to-head difference, full-domain tolerance convergence,
matched-tolerance readout on both the census and the model responses, the payoff-table
noise decomposition, the rollout aggregates, and confirmatory population arithmetic.
Every numerical claim covered by the verifier manifest is compared with an absolute
bound at half of the last digit the paper displays, so a changed covered digit cannot
pass. Claims backed by released raw tensors are recomputed; claims whose release
contains only summaries receive arithmetic, provenance, and cross-file consistency
checks rather than an independent replay of the omitted responses. Two such checks
span files: the offline reanalysis must reproduce the sealed formal baseline exactly,
and both analyses must name the same 120 selected states.

A compact verification consists of running the paper-number script, running the
common-evaluator analysis together with the shared-deck evaluator tests and the
packaged invariant tests, and compiling this document. Verification from the released
bundle alone does not independently replay the 235,924 model decisions, because the
manifest omits the workspace items and raw responses; re-running the local models
requires the corresponding weights and Ollama environment together with those
item-generation assets.

\end{document}